\documentclass{article}

\usepackage{microtype}
\usepackage{graphicx}
\usepackage{booktabs} 
\usepackage{amsfonts, amssymb, bm, bbm, amsthm}
\usepackage{hyperref}

\usepackage[accepted]{icml2024}

\usepackage{amsmath}
\usepackage{amssymb}
\usepackage{mathtools}
\usepackage{amsthm}

\usepackage[capitalize,noabbrev]{cleveref}

\theoremstyle{plain}
\newtheorem{theorem}{Theorem}[section]

\newtheorem{lemma}[theorem]{Lemma}
\newtheorem{corollary}[theorem]{Corollary}
\theoremstyle{definition}
\newtheorem{definition}[theorem]{Definition}
\newtheorem{assumption}[theorem]{Assumption}
\theoremstyle{remark}
\newtheorem{remark}[theorem]{Remark}
\usepackage[most]{tcolorbox}
\usepackage{subcaption}

\usepackage{esvect}
\usepackage[T1]{fontenc}
\newtheorem{propo}[theorem]{Proposition}

\newtheoremstyle{myexample} 
    {\topsep}                    
    {\topsep}                    
    {\rm }                   
    {}                           
    {\bf }                   
    {.}                          
    {.5em}                       
    {}  

\newtheoremstyle{myremark} 
    {\topsep}                    
    {\topsep}                    
    {\rm}                        
    {}                           
    {\bf}                        
    {.}                          
    {.5em}                       
    {}  

\newcommand{\bq}{{\boldsymbol{q}}}

\def\R{\mathbb{R}}

\def\bp{{\boldsymbol p}}

\def\0{{\boldsymbol 0}}
\def\1{{\boldsymbol 1}}

\def\1{{\boldsymbol 1}}

\newcommand{\E}{\operatorname{\mathbb{E}}}

\usepackage[textsize=tiny]{todonotes}

\icmltitlerunning{Conformal Prediction with Learned Features}

\begin{document}

\twocolumn[
\icmltitle{Conformal Prediction with Learned Features}




\begin{icmlauthorlist}
\icmlauthor{Shayan Kiyani}{yyy}
\icmlauthor{George Pappas}{yyy}
\icmlauthor{Hamed Hassani}{yyy}

\end{icmlauthorlist}

\icmlaffiliation{yyy}{The Electrical and Systems Engineering Department, University of Pennsylvania, University of Pennsylvania, USA}

\icmlcorrespondingauthor{Shayan Kiyani}{shayank@seas.upenn.edu}
\icmlcorrespondingauthor{George Pappas}{pappasg@seas.upenn.edu}
\icmlcorrespondingauthor{Hamed Hassani}{hassani@seas.upenn.edu}

\icmlkeywords{Machine Learning, ICML}

\vskip 0.3in
]



\printAffiliationsAndNotice{} 

\begin{abstract}
In this paper, we focus on the problem of conformal prediction with conditional guarantees. Prior work has shown that it is impossible to construct nontrivial prediction sets with full conditional coverage guarantees. A wealth of research has considered relaxations of full conditional guarantees, relying on some \emph{predefined} uncertainty structures. Departing from this line of thinking, we propose Partition Learning Conformal Prediction (PLCP), a framework to improve conditional validity of prediction sets through \emph{learning} uncertainty-guided features from the calibration data. We implement PLCP efficiently with alternating gradient descent, utilizing off-the-shelf machine learning models. We further analyze PLCP theoretically and provide conditional guarantees for infinite and finite sample sizes. Finally, our experimental results over four real-world and synthetic datasets show the superior performance of PLCP compared to state-of-the-art methods in terms of coverage and length in both classification and regression scenarios.

\end{abstract}

\section{Introduction} \label{sec: introduction}
Consider a distribution $\mathcal{D}$ over a domain $\mathcal{X} \times \mathcal{Y}$,  where  $\mathcal{X}$ denotes the space of covariates and $\mathcal{Y}$ denotes the space of labels. Let  $f:\mathcal{X}\rightarrow \mathcal{Y}$ be a (pre-trained) \emph{model} that provides for every input $x$ a \emph{point estimate} of the corresponding label $y$.  Using the model $f$ and  a set of new calibration samples $\left(X_1, Y_1\right), \ldots,\left(X_n, Y_n\right)$, generated i.i.d. from $\mathcal{D}$, the goal of conformal prediction is  to construct for every input $x$ a \emph{prediction set} $C(x)$ that is guaranteed to cover the true label $y$ with high probability. Formally, we say that the prediction sets $C(x) \subseteq \mathcal{Y}$ have \emph{marginal} coverage guarantee if for a test sample $(X_{n+1}, Y_{n+1})$ we have
\begin{equation} \label{marginal}
\Pr\left(Y_{n+1} \in C\left(X_{n+1}\right)\right) = 1-\alpha,
\end{equation}
where $\alpha$ is the miscoverage rate, and the probability is taken over the randomness in  calibration and test points. 

Oftentimes in practice, methods that only guarantee marginal coverage fail to provide valid coverage with respect to specific subgroups or under changing conditions \cite{Romano2020With, Guan2021LocalizedCP, lei2014distribution}. This issue is particularly evident in applications such as healthcare, where obtaining valid prediction sets for different patient demographics is crucial. For instance, a marginal method might perform well on average over new patients but fail to construct accurate prediction sets for certain age groups or medical conditions. 

Ultimately, we may seek to construct prediction sets that achieve \emph{full conditional coverage} which requires for every $x \in \mathcal{X}$ 
\begin{align}\label{full}
    \Pr\left(Y_{n+1} \in C\left(X_{n+1}\right) \mid X_{n+1}=x\right)=1-\alpha.
\end{align}
 
Despite the importance of achieving conditional guarantees, there are some fundamental limitations.  Prior work \cite{pmlr-v25-vovk12, 2019arXiv190304684F, lei2014distribution} has shown that it is impossible to construct nontrivial prediction sets with distribution-free, full conditional coverage when we have access to a finite-size calibration set. Consequently, relaxations of \eqref{full} have been considered. For instance, \cite{gibbs2023conformal,Tibshirani2019ConformalPU} develop frameworks to guarantee coverage under a predefined class of covariate shifts.  Another line of work \cite{Romano2020With, jung2023batch, Barber2019TheLO} considers predefined groups of the covariates and guarantees coverage conditioned on those groups. A more detailed discussion on the existing methods and their implications is provided in the related works section \ref{Rel}.

In this paper, we take a new approach and, instead of considering predefined structures, propose to \emph{learn} structures from the calibration data    that are \emph{informative about uncertainty quantification}. Our algorithmic framework aims at learning such structures in conjunction with constructing the prediction sets in an iterative fashion. To better illustrate our approach and contributions, we will proceed with the following toy example. 

\begin{figure}[ht] 
\centering
\subcaptionbox{}{\includegraphics[width=0.49\columnwidth]{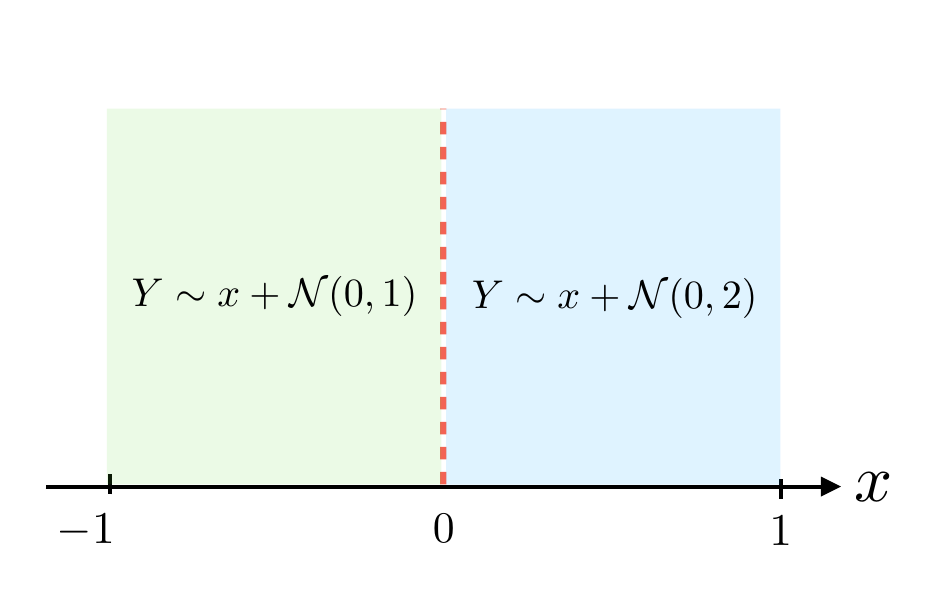}}
\hfill 
\subcaptionbox{}{\includegraphics[width=0.49\columnwidth]{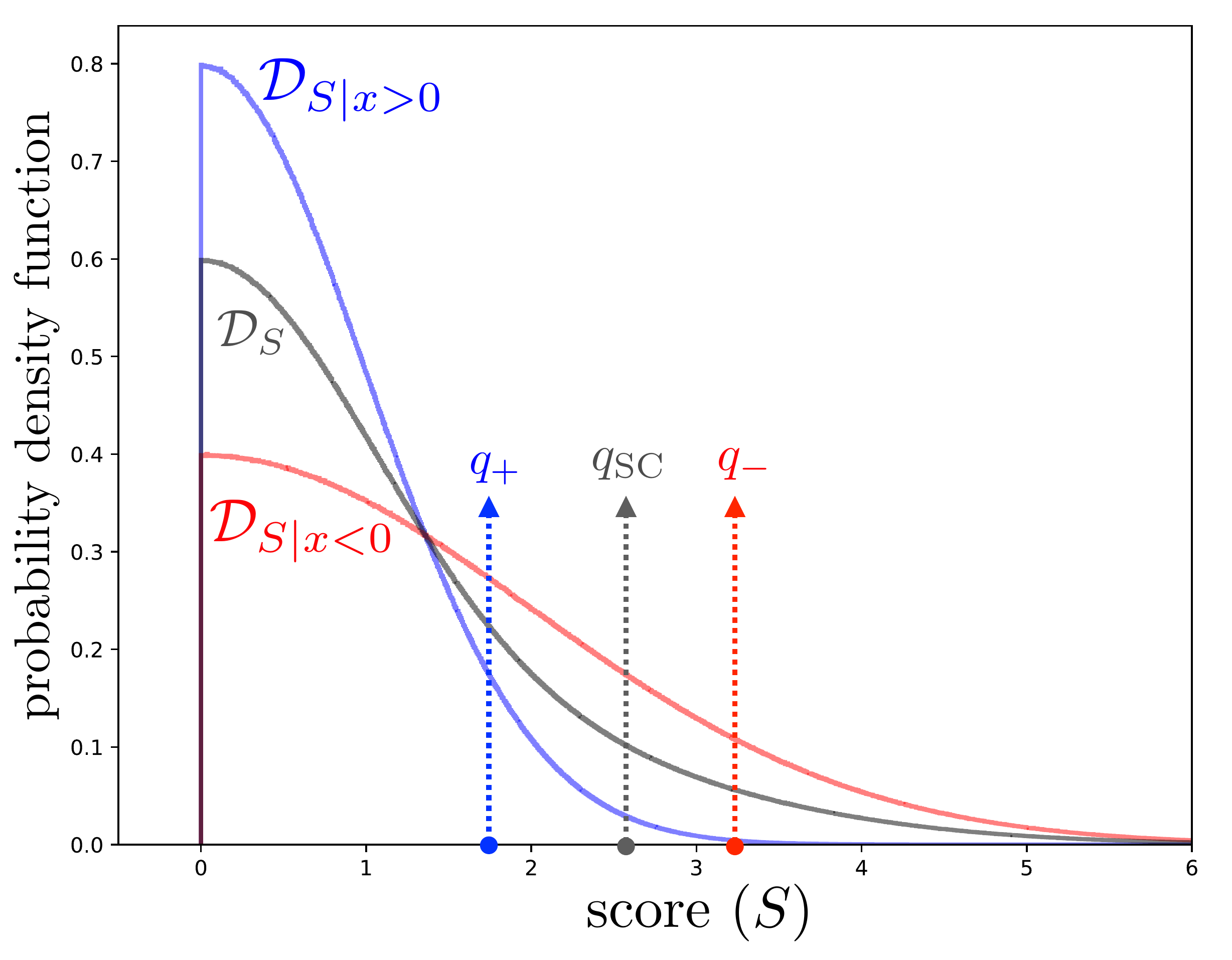}}

\vskip 0.1in 

\subcaptionbox{}{\includegraphics[width=0.50\columnwidth]{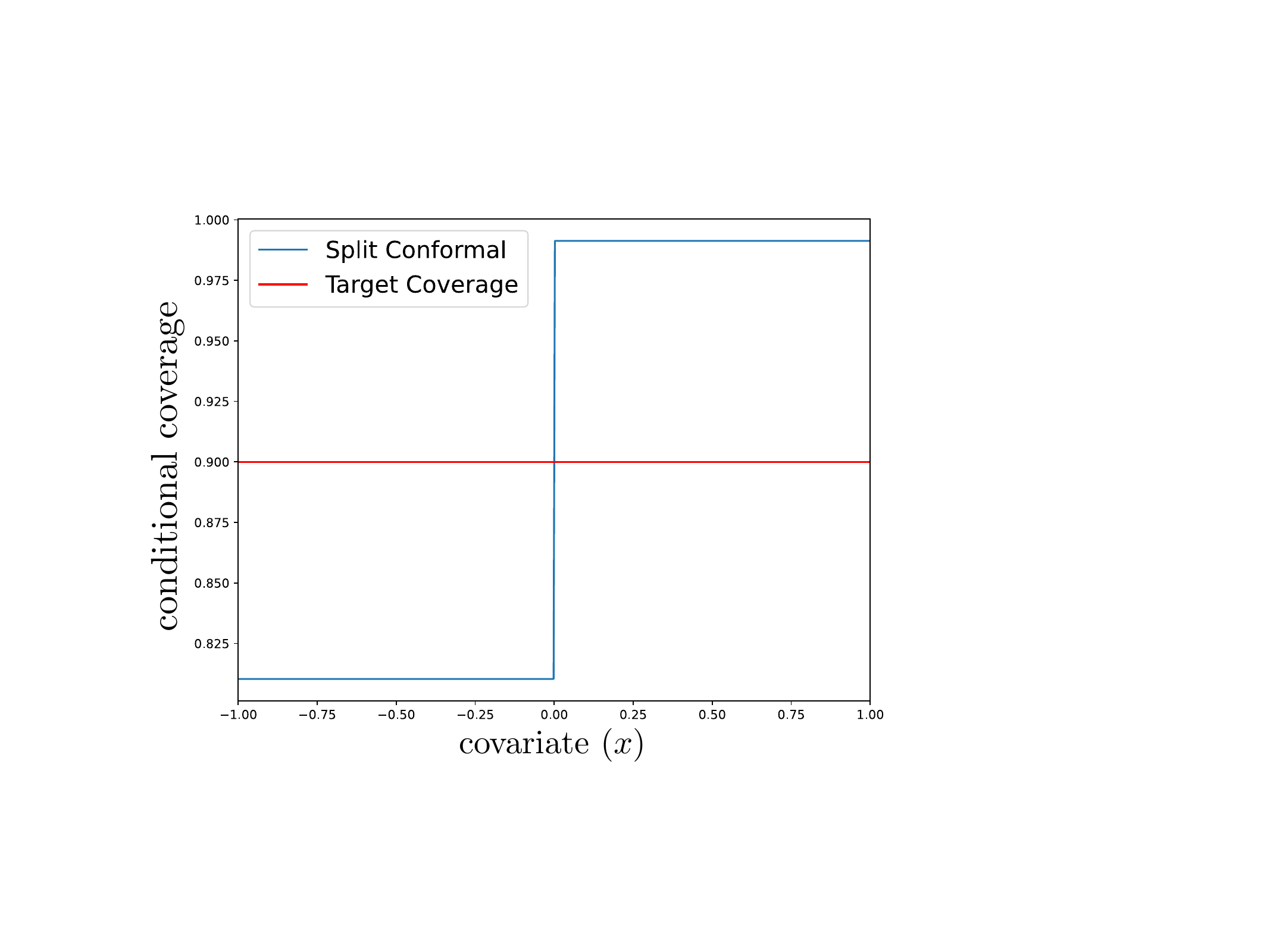}}
\hfill 
\subcaptionbox{}{\includegraphics[width=0.48\columnwidth]{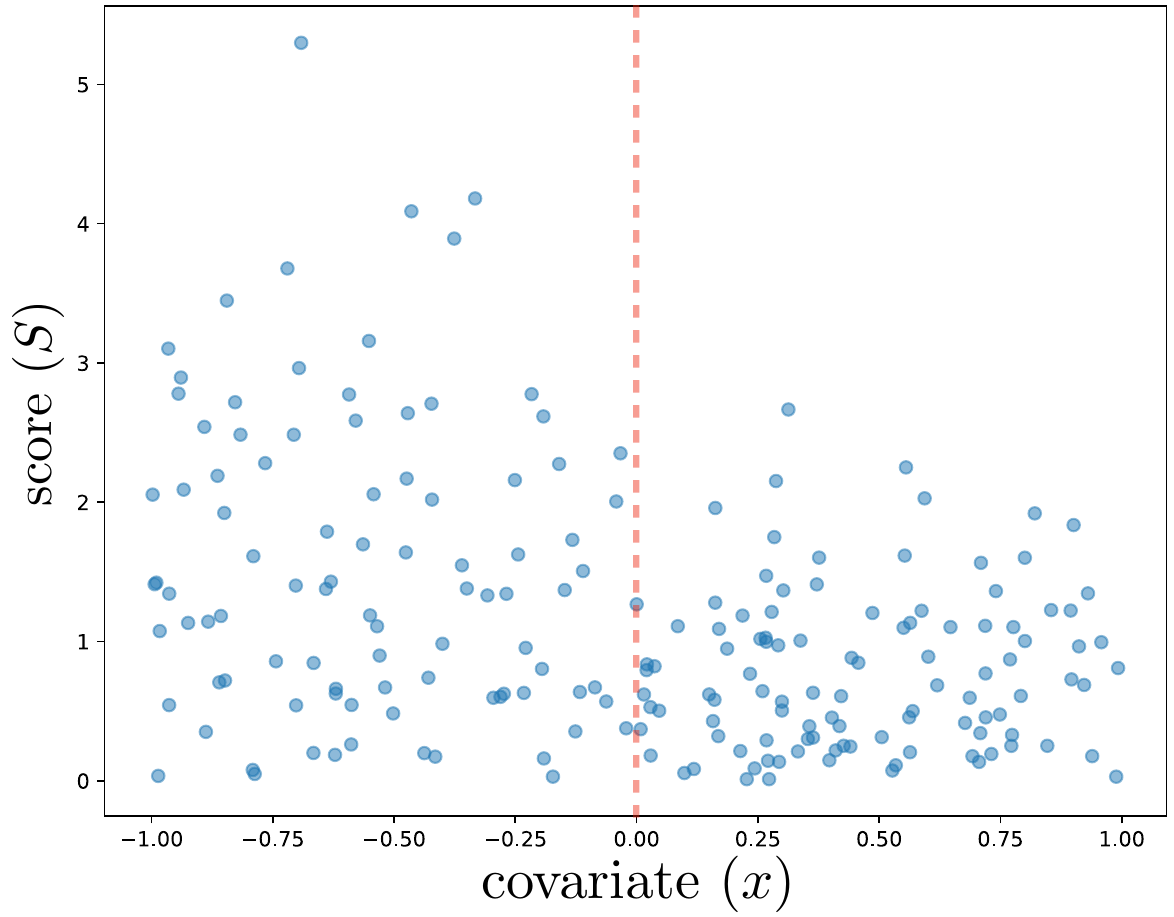}}

\caption{(a) Distribution of the labels conditioned on the covariate ($x$). (b) Conditional and marginal distributions of the score. (c) Coverage of the prediction sets of Split Conformal conditioned on $x$. (d) Samples of the form $(x_i,s_i)$. From these samples, we aim to learn the partition/feature $h$ of the covariate space shown by the dashed red line.}
\label{fig:example}
\end{figure}





\noindent\textbf{Example.} Let $\mathcal{X} = [-1,+1]$ and assume that the covariate $X$ is uniformly distributed over $\mathcal{X}$. The label $Y$ is distributed as follows (see also Figure~\ref{fig:example}-(a)): 
\begin{itemize}
\item[-] If $x < 0 $, then $Y \sim x +   \mathcal{N}(0,1)$,
\item[-] If $x \geq 0$, then $Y    \sim  x + \mathcal{N}(0,2)$.
\end{itemize}
For simplicity, we assume in this example that we have infinitely many data points available from this distribution. It is easy to see that the optimal\footnote{Here, optimality is measured in terms of the mean squared error.} regression model for this  distribution is $ f(x)=\mathbb{E}[Y \mid X = x] = x$. As a result, considering the conformity score $S(x,y) = | y - f(x) | $, the distribution of $S$ follows the folded Gaussian distribution: i.e. for $x < 0$ we have $\mathcal{D}_{S | x} = |N(0,1)|$, and for $x \geq 0$ we have $\mathcal{D}_{S | x} = |N(0,2)|$. Figure~\ref{fig:example}-(b) depicts the conditional and marginal distributions of $S$  with their $0.9$-quantiles. 

In this setting, the standard Split Conformal method only looks at  the marginal distribution of $S$ and constructs the prediction sets according to its quantile value $q_{\rm SC}$. This choice of prediction sets, although being marginally valid as in \eqref{marginal}, will lead to over-coverage when $x < 0$, and under-coverage when $x \geq 0$ (see Figure~\ref{fig:example}-(c)). For this specific example, this limitation can be overcome by constructing the prediction sets based on the sign of the covariate. Let $q_+$  (rep. $q_{-}$) be the $(1-\alpha)$-quantile of the distribution of $S$ conditioned on $x \geq 0$ (rep. $x < 0$). Then the following prediction sets will provide a valid coverage on both positive and negative covariates (in fact, here we have full conditional coverage as in \eqref{full}):   
\begin{itemize}
\item[-] If $x < 0 $, then $C(x) = \{y \in \mathbb{R} \,\text{ s.t. } \, S(x,y) \leq q_- \}$,
\item[-] If $x \geq 0$, then $C(x) = \{y  \in \mathbb{R} \,\text{ s.t. } \, S(x,y) \leq q_+ \}$.
\end{itemize} 
A few points about the above example are in order: (i) This example showcases how we can obtain richer, conditionally-valid prediction sets by using an appropriate partitioning of the covariate space, and constructing the prediction sets according to the partitions. Put differently, the prediction sets were obtained using a new feature, $h(x) = \text{sign}(x)$, of the covariates. (ii) Further, in the above example, $f(x) = x$ is the optimal pointwise prediction of the label. Thus, the feature $h$ is not useful in obtaining a more accurate pointwise estimate of the label; but rather, it is useful in quantifying the \emph{uncertainty} in estimating the label. (iii) Finally, in reality informative features  such as $h$ are not known apriori and they have to be \emph{learnt} from data. This is a challenging task as we need to find features, or boundaries in the covariate space, that can separate points $x,x' \in \mathcal{X}$ based on how different the conditional distributions $\mathcal{D}_{S \mid x}$ and $\mathcal{D}_{S \mid x'}$ are. Yet, these conditional distributions are not known and only a finite number of samples of the form $(x_i, s_i)$ are given (see Figure~\ref{fig:example}-(d)).  

In light of the example above, our goal is to identify, in a data-driven manner, a partitioning of the covariate space such that points in the same partition share some similarities in terms of their prediction sets. For example, for two points in the same partition, we would like their respective quantile of the conditional distribution of the score to be close to each other. The main challenge here is that the distributions are unknown and only a finite-size calibration set is given.

\textbf{Algorithmic Contributions.} We formalize the problem of learning such uncertainty-guided partitions in Section~\ref{algorithm}. We will first derive an optimization objective that measures the quality of a partitioning based on how far the conditional quantiles of the points in that partition differ from each other. Our main algorithm, PLCP, optimizes this objective by iteratively updating the partitioning and prediction sets over a given calibration data set. The partitioning is chosen over a function class, e.g. linear functions or neural networks.
In this sense, PLCP can systematically utilize off-the-shelf machine learning models very efficiently to construct the prediction sets.

\textbf{Theoretical Contributions.} We introduce the notion of ``Mean Squared
Conditional Error (MSCE)'' defined as 
    \begin{align}\label{MSCE}\text{MSCE}(\mathcal{D}, \alpha, C) = \E\left[ \left(\text{cov}(X) - (1 - \alpha) \right)^2 \right]\quad \end{align}
    $$\text{ where: }\quad\text{cov}(x) = \Pr[Y \in C(X) \mid X = x],$$
which measures the deviation of  prediction sets $C(x)$ from \eqref{full}. For the ease of notation, we will drop the arguments inside $\text{MSCE}(\mathcal{D}, \alpha, C)$ and simply use MSCE. 
In section \ref{thm}, we establish two insightful theoretical guarantees for the prediction sets constructed by PLCP. In the infinite data regime, we demonstrate that the MSCE of PLCP's prediction sets scales as $O(\frac{1}{\sqrt{m}})$, where $m$ denotes the number of regions in the learned partitioning of the covariate space. Hence, as the number of partitions increases, the deviation from the nominal coverage, $1-\alpha$, diminishes across the covariate space. Further, we present a finite sample guarantee for the MSCE of the prediction sets obtained by PLCP. We show that the MSCE scales as $O\left(\sqrt{\frac{m\log{n} + \text{complexity}(\mathcal{H})}{n}} + \frac{1}{\sqrt{m}}\right)$, with high probability, where $\mathcal{H}$ represents the class of functions used by PLCP, and $\text{complexity}(\mathcal{H})$ quantifies its complexity (i.e. the covering number). Such PAC-style guarantees are common in the conformal prediction literature \cite{pmlr-v25-vovk12, Park2020PAC, jung2023batch}. Finally, at the end of the section we provide implied coverage guarantees (both marginal and conditional) for PLCP.  

\textbf{Experimental Results.} 
Across four diverse datasets and tasks, we compared our method with established approaches such as Split Conformal \cite{Lei2016DistributionFreePI, Papadopoulos2002InductiveCM}, CQR \cite{romano2019conformalized}, LocalCP \cite{hore2023conformal}, Batch-GCP \cite{jung2023batch}, and Conditional Calibration \cite{gibbs2023conformal}. PLCP consistently outperformed Split Conformal in terms of conditional coverage and interval length. Unlike BatchGCP, which relies on predefined  groups (e.g., race, gender), our method requires no such prior knowledge. In the experiments, PLCP matched BatchGCP's performance on known groups and effectively identified and covered additional meaningful groups. Compared to Conditional Calibration, PLCP achieved comparable coverage but with notably shorter prediction intervals. This success is attributed to PLCP's effective integration of advanced machine learning models to extract relevant features for uncertainty quantification.

\subsection{Related Work} \label{Rel}
We begin by highlighting some of the key developments in obtaining conditional guarantees. 

\textbf{Group-Conditional Methods.} These methods consider a \emph{predefined} collection of groups $\mathcal{G}$ within the covariate space and guarantee coverage conditioned on each group $g \in G$ \cite{jung2023batch, Barber2019TheLO, vovk2003mondrian, javanmard2022prediction}. Our approach departs significantly from such methods as it  \emph{learns} a partitioning of the covariate space directly from the calibration data, not pre-establishing them, in order to identify diverse behaviors in $\Pr[Y|X]$. This distinction is crucial in cases with limited knowledge of the uncertainty of the pre-trained predictive model on varying subpopulations. 


\textbf{Covariate Shifts.} An alternative approach to relax \eqref{full} is to provide conditional guarantees with respect to a pre-defined family of covariate shifts  \cite{gibbs2023conformal, Tibshirani2019ConformalPU, cauchois2023robust}. The set of covariate shifts are fixed prior to the calibration phase, and hence they may not capture covariate shifts tied to the uncertainty of the pre-trained predictive model. \citet{gibbs2023conformal} aims to mitigate this by expanding the covariate shift space, but this can lead to conservatism, creating unnecessarily large prediction sets, as pointed out by \citet{hore2023conformal}. In contrast, our method learns the uncertainty patterns of the predictive model during calibration, leading to smaller and more precise prediction sets as demonstrated in Section~\ref{exp}. We provide a detailed comparison with Conditional Calibration \cite{gibbs2023conformal} in Remark \ref{candes}, in the Appendix.


\textbf{Designing Better Conformity Scores.} An alternative approach in the conformal inference literature to improve conditional validity is to design new conformity scores \cite{Lei2016DistributionFreePI, romano2019conformalized, chernozhukov2021distributional, deutschmann2023adaptive, feldman2021improving, romano2020classification}. These scores aim to better encapsulate the complexities inherent in $\Pr[Y|X]$ in different applications. In contrast, in the pipeline of conformal prediction, our method applies after selecting the conformity score. This would enable us to capture more sophisticated uncertainty patterns of the pre-trained predictive model. Nonetheless, score-selection methods can potentially be integrated into our framework.

\textbf{Further Works.}
\citet{Guan2021LocalizedCP,hore2023conformal} aim at relaxing \eqref{full} using a predefined pairwise similarity function on the covariate space. The choice of these functions is generally heuristic and rather less dependent on the data. Moreover, \citet{pmlr-v108-izbicki20a, leroy2021md, izbicki2022cd,amoukou2023adaptive} aim to estimate the distribution of scores using calibration data. Such methods can become inapplicable in modern high-dimensional datasets as accurate estimation of the whole conditional distribution of scores is not possible from calibration data. In contrast to all these works, we aim at systematically learning low-dimensional features from the calibration data that characterizes the amount of uncertainty in the labels. In section \ref{exp}, we showcase PLCP's ability to success at high dimensional datasets including image data. On a parallel thread, other works also looked at class-conditional (conditioned on $Y$) coverage guarantees \cite{ding2024class, si2023pac}. While both problems, i.e. class conditional coverage and covariate conditional coverage, are pertinent to the field of conformal prediction, they represent distinct challenges and have developed along separate trajectories.

\section{Algorithm}\label{algorithm}
\textbf{Preliminaries and Notations.} Recall that given a calibration set $\{(X_i, Y_i)\}_{i=1}^n$ consisting of i.i.d. samples from a distribution $\mathcal{D}$ on $\mathcal{X} \times \mathcal{Y}$, and a (pre-trained) model $f$, conformal prediction aims to construct for every input $x \in \mathcal{X}$ a prediction set $C(x)$ that is guaranteed to cover the true label $y$ with high probability. These sets are often constructed using a given conformity score $S: \mathcal{X} \times \mathcal{Y} \rightarrow \mathbb{R}$. For example, one commonly used conformity score for regression tasks is $S(x,y) = | y - f(x) |$. Define the random variable $S = S(X,Y)$, where $(X,Y) \sim \mathcal{D}$. To keep the notation simple, we use the notation $\mathcal{D}$ to also refer to the distribution of $(X,S)$.  Furthermore, we use $\mathcal{D}_S$ to denote the marginal distribution of $S$, and $\mathcal{D}_{S | x}$ to denote the conditional distribution of $S$ given an input $x \in \mathcal{X}$. Denoting $S_i = S(X_i, Y_i)$, note that  $\{(X_i, S_i)\}_{i=1}^n \stackrel{\text{i.i.d.}}{\sim} \mathcal{D}$.

One standard approach to conformal prediction is the so-called Split Conformal method (SC), which lets 
$$C_{\rm SC}(x) = \{ y \in \mathcal{Y} | S(x,y) \leq q_{\rm SC}\},$$ where $q_{\rm SC}$ is the $\lceil (1-\alpha) n\rceil$-th largest element in the set $\{S_1, S_2, \cdots, S_n, \infty\}$. The quantity $q_{\rm SC}$ can be considered as an estimate of the $(1-\alpha)$-quantile of the marginal distribution of $S$, i.e. $\mathcal{D}_S$.  In this way, the prediction sets are mainly constructed based on the distribution $\mathcal{D}_S$, and the information from the covariates $X_i$ is not used directly. As a consequence, the Split Conformal method (and its variants) can only provide \emph{marginal} coverage guarantees.

In what follows, we will use the pinball loss defined as 
\begin{equation} \label{pinball_loss}
\ell_\alpha(q, s)= 
\begin{cases}
\alpha(q-s) & \text{if } q \geq s, \\
(1-\alpha)(s-q) & \text{if } q < s.
\end{cases}
\end{equation}
It is well known that, for random variable $S\sim \mathcal{D}_S$, minimizing the expected pinball loss over $q$ yields the $(1-\alpha)$-quantile of the distribution, which we denote by $q_{1-\alpha}(S)$.  I.e. $q_{1-\alpha}(S) =  \underset{q\in \R}{\text{argmin}}\: \E_{S \sim \mathcal{P}}\ell_\alpha(q, S)$. 

We will use bold symbols like $\bq$ for vectors. For any $\bq = (q_1, q_2, \cdots, q_m) \in \mathbb{R}^m$ and probability vector $\bp = (p_1, \cdots, p_m)$, we will use the notation $\bq_{i\sim\bp}$, to point to the random variable that takes the value $q_i$ with probability $p_i$.

\textbf{Algorithmic Principles.} As mentioned in Sec.~\ref{sec: introduction}, at a high level, our main goal is to learn a partitioning of the covariate space $\mathcal{X}$, which is representative of the uncertainty structure of the label, and construct the prediction sets accordingly. For instance, assume that we would like to partition $\mathcal{X}$ into two groups. Intuitively speaking, we would like one of the groups to include points $x$ such that $\mathcal{D}_{Y | x}$ is ``more'' noisy, and the other group to include points $x$ such that $\mathcal{D}_{Y | x}$ is ``less'' noisy. For more noisy $x$'s, the distribution $\mathcal{D}_{S | x}$ takes higher values, and hence its $(1-\alpha)$-quantile is high. And for the less noisy $x$'s, the $(1-\alpha)$-quantile of  $\mathcal{D}_{S | x}$ is low. 

To better illustrate the algorithmic principles, let us first assume that we have infinite data. We will shortly focus on the finite-size setting for which our algorithm is designed. 
We would like to partition the covariate space $\mathcal{X}$ into $m$ groups $G_1, \cdots, G_m$.
Fix a group $G_i$. The prediction set for each $x \in G_i$ will take the same form, i.e. $S(x,y) \leq q_{i}$. Hence, our goal will be to find both the groups $G_i$ and the values $q_i$ (i.e. the prediction sets). 

 Let us first see how $q_i$'s could be found depending on the choice of $G_i$'s. In order to guarantee coverage conditioned on the group $G_i$, $q_i $ can simply be chosen as the $(1-\alpha)$-quantile of the distribution of $S$ conditioned on $G_i$, i.e.
 \begin{equation} \label{q_vs_G}
 q_i = q_{1-\alpha}(S | X \in G_i).
 \end{equation}
 Thus, fixing the groups $G_1, \cdots, G_m$, one can compute the values $q_i$ as the corresponding quantiles and construct the prediction sets:
 \vspace{-.2cm}
\begin{equation}
C(x)= 
\begin{cases}
\{y\in\mathcal{Y}| S(x,y)\leq q_{1}\} & \text{if } x \in G_1, \\
\quad\quad\quad\vdots & \quad\vdots\\
\{y\in\mathcal{Y}| S(x,y)\leq q_{m}\} & \text{if } x \in G_m.
\end{cases}
\vspace{-.1cm}
\end{equation}
Now, assume that the values $q_1, \cdots, q_m$ are given and we would like to find the groups $G_i$ accordingly. 
Ideally, we want to find the groups $G_i$ in a way that if $x \in G_i$, then the $(1-\alpha)$-quantile of the conditional distribution $\mathcal{D}_{S|x}$, i.e. $q_{1-\alpha}(S | X = x)$, is very close to $q_i$. But this may not be possible given the specific values of $q_i$'s. Instead, our key insight is to assign $x$ to the group $G_i$ whose associated value $q_i$ is closest to $q_{1-\alpha}(S | X = x)$. To quantify closeness, note that $q_{1-\alpha}(S | X = x)$ is written using the pinball loss \eqref{pinball_loss}: 
\begin{align*} 
q_{1-\alpha}(S|X=x) = 
\underset{q\in \R}{\text{argmin}}\: \E_{ S|X=x}\ell_\alpha(q, S) \quad \forall x\in \mathcal{X}.
\end{align*}
Accordingly, the quantity $\E_{ S|X=x}  \ell_\alpha(q, S) $ measures how close a value $q$ is to the quantile $q_{1-\alpha}(S|X=x)$. This measure can also be interpreted in terms of conditional coverage. Note that 
$\frac{d}{d q} \E_{ S|X=x} \left[ \ell_\alpha(q, S) \right] = \Pr\{ S(X,Y) \leq q | X = x\} - (1-\alpha)$. As a result, $\E_{ S|X=x} \left[ \ell_\alpha(q, S) \right]$ can be considered as a measure of miscoverage for the prediction set $S(x,y) \leq q$. Using this measure, the closest point $q_{i^*}$ to $q_{1-\alpha}(S | X = x)$ can be found as
\begin{align}\label{G_vs_q}
    i^* = \underset{i \in [1, \cdots, m]}{\text{argmin}} \E_{ S|X=x}\ell_\alpha(q_i, S).
\end{align}
In summary, we have identified two algorithmic principles to derive the groups $G_i$ and the values $q_i$ (i.e. the prediction sets). Given the groups $G_i$, choose $q_i$ using \eqref{q_vs_G}; and given the $q_i$'s, for any $x \in \mathcal{X}$ assign its group according to \eqref{G_vs_q}.  

\textbf{The Algorithm.} We will now proceed with implementing principles \eqref{q_vs_G} and \eqref{G_vs_q} in the finite-size setting. One way to do this is to derive an equivalent optimization objective which admits an unbiased estimate using finite samples. 

Consider the following optimization problem over the variable $\bq = (q_1, q_2, \cdots, q_m) \in \mathbb{R}^m$:
\begin{align}\label{primary}
\vspace{-.2cm}
    \bq^\infty = \underset{\bq = (q_1, \cdots, q_m) \in \R^m}{\text{argmin}}\; \E_X \left[ \underset{i \in [1, \cdots, m]}{\text{min}}\E_{ S|X=x}\ell_\alpha(q_i, S) \right].
    \vspace{-.2cm}
\end{align}
The above optimization problem is convex as the pinball loss is convex. Further, by defining $G^{\infty}_i$ according to \eqref{G_vs_q}: 
\begin{equation}
G^{\infty}_i = \left\{x \in \mathcal{X} \text{ s.t. } i = \underset{t \in [1, \cdots, m]}{\text{argmin}} \E_{ S|X=x}\ell_\alpha(q^\infty_t, S)  \right\},
\end{equation}
it is easy to see that the  pair $\bq^\infty$ and $\{G_i^\infty\}_{i=1}^m$ satisfy both \eqref{q_vs_G} and  \eqref{G_vs_q}. Hence, we will focus on \eqref{primary}.  Let us rewrite \eqref{primary} in a slightly different but equivalent form. Let $\Delta_m$ be the $m$-dimensional simplex; i.e.  the set of all the probability vectors $\bp = (p_1, \cdots, p_m) \in \mathbb{R}^m$. We have
\vspace{-.2cm}
\begin{align*}
   \underset{i \in [1, \cdots, m]}{\text{min}} \!\E_{S|X=x}\ell_\alpha(q_i, S) \!  = \!\!
   \underset{\bp \in \Delta_m}{\text{min}} \sum_{i=1}^m p_i\E_{ S|X=x}\ell_\alpha(q_i, S).
\end{align*}
\vspace{-.1cm}
Using the above relation, we can rewrite \eqref{primary} as:
\vspace{-.2cm}
\begin{align*}
    \bq^\infty&=\underset{\bq\in \R^m}{\text{argmin}}\; \E_X \underset{\bp \in \Delta_m}{\text{min}} \sum_{i=1}^m p_i\E_{ S|X=x}\ell_\alpha(q_i, S)\\
    &= \underset{\underset{h: \mathcal{X} \to \Delta_m}{\bq\in \R^m}}{\text{argmin}}\; \E_X  \sum_{i=1}^m h^i(x)\E_{ S|X=x}\ell_\alpha(q_i, S)\\
    &= \underset{\underset{h: \mathcal{X} \to \Delta_m}{\bq\in \R^m}}{\text{argmin}}\; \E_{(X, S)}  \sum_{i=1}^m h^i(x)\ell_\alpha(q_i, S).
    \vspace{-.2cm}
\end{align*}
Here, in the second and third step the minimization is over $\bq$ and all the functions $h = (h^1, \cdots, h^m) : \mathcal{X} \to \Delta_m$.  The second equality follows from the fact that the function 
\vspace{-.1cm}
$$\bp^\infty(x)=\underset{\bp \in \Delta_m}{\text{argmin}} \sum_{i=1}^m p_i\E_{ S|X=x}\ell_\alpha(q_i, S),$$ 
is a mapping from $\mathcal{X}$ to $\Delta_m$. Our  optimization problem can thus be written as
\vspace{-.2cm}
\begin{align}\label{secondary}
    h^\infty, \bq^\infty &= \underset{\underset{h: \mathcal{X} \to \Delta_m }{\bq\in \R^m}}{\text{argmin}} \underset{(X, S) \sim \mathcal{D}}{\mathbb{E}}\left[\sum_{i=1}^m h^i(X)\ell_\alpha(q_i, S)\right].
    \vspace{-.2cm}
\end{align}
With finite-size data, we  perform the following two common relaxations on the above objective: (i) Replace the objective with its empirical version (i.e. replace the expectation with the sum over the calibration data); (ii) Instead of optimizing over all the functions $h: \mathcal{X} \to \Delta_m$, which clearly overfits in the finite-size setting, we optimize over a function class $\mathcal{H}$. E.g., $\mathcal{H}$ could be the class of linear functions or neural networks with a soft-max layer at the output. Our final optimization problem, using the finite-size calibration set $\{(X_i, S_i)\}_{i=1}^n$ and function class $\mathcal{H}$,  becomes: 
\begin{align}\label{final}
    h^*, \bq^* &= \underset{\underset{h \in \mathcal{H} }{\bq\in \R^m}}{\text{argmin}} \frac{1}{n}\sum_{j=1}^n \sum_{i=1}^m h^i(X_j)\ell_\alpha(q_i, S_j).
\end{align}
Partition Learned Conformal Prediction (PLCP) algorithm is formulated utilizing \eqref{final}, in Algorithm \ref{PLCP}.
\begin{algorithm}
\caption{Partition Learned Conformal Prediction (PLCP)}
\begin{algorithmic}[1]
    \REQUIRE Data: $\{(X_i, Y_i)\}_{i=1}^n$, Conformity score: $S(x, y)$, Number of groups: $m$, Family of functions: $\mathcal{H}$
    
    \STATE Compute $S_i = S(X_i, Y_i),\quad \forall i \in [1, \cdots, n]$.
    \STATE Solve the optimization problem,\\ 
        $$h^*, \bq^* =  
        \underset{\underset{h \in \mathcal{H}}{\bq\in \R^m}}{\text{argmin}} \frac{1}{n}\sum_{j=1}^{n} \sum_{i=1}^{m} h^i(X_j) \ell_{\alpha}(q_i, S_j).$$
    \ENSURE Prediction set $C^*(x) = \{ y \mid 
    S(x, y) \leq \bq^*_{i\sim h^*(x)}\}$
\end{algorithmic}
\label{PLCP}
\end{algorithm}

\begin{remark}
    Oftentimes in practice, the machine learning models, such as Neural Networks, are parametric ($h_\theta$, where $\theta$ is the set of parameters). In that case, we can simply implement PLCP with alternating gradient descent, performing few steps of gradient descent on $q$ and $\theta$ at each iteration.
\end{remark}



\vspace{-.1cm}
\section{Theoretical Results}\label{thm}
We introduced in \eqref{MSCE} a new notion to measure conditional coverage called the mean squared coverage error (MSCE). In words, MSCE penalizes the deviation of the coverage of prediction sets, conditioned on each point $x$, from the nominal value $1-\alpha$. Further, as we will show, any bound on the MSCE can be simply translated into a bound on coverage. 
As a result, minimizing MSCE can be a valid objective as a relaxation of \eqref{full}. In what follows, we provide guarantees for the MSCE of PLCP in the presence of finite and infinite data. Fallback coverage guarantees will also be provided at the end of this section. 
All the proofs has been moved to the Appendix \ref{Proofs}. First, let us introduce our main assumptions. 

\begin{assumption}\label{ass:iid}
The set $\{(X_i, Y_i)\}_{i=1}^n$ are generated  i.i.d. 
\end{assumption}

\begin{assumption}\label{ass:bounded}
    The conformity score $s(.,.)$ should be bounded. Without loss of generality, we can assume~$s(\cdot,\cdot)\in [0, 1]$.
\end{assumption}

\begin{definition}
    A distribution $\mathcal{P}$, is called $L$-lipschitz if we have for every real valued numbers $q \leq q^{\prime}$,
        $$\Pr_{X \sim \mathcal{P}}\left(X \leq q^{\prime}\right)-\Pr_{X \sim \mathcal{P}}(X \leq q) \leq L\left(q^{\prime}-q\right).$$
\end{definition}

\begin{assumption}\label{ass:reg}
    The conditional distribution $\mathcal{D}_{S|X}$ should be $L$-Lipschitz, almost surely with respect to $\mathcal{D}_{X}$.
\end{assumption}
Assumption \ref{ass:iid} is often necessary to obtain concentration guarantees and has been used in the literature of  of conformal prediction \cite{Sesia2021ConformalPU, jung2023batch, Lei2012DistributionFP, Guan2021LocalizedCP}. These works have also considered regularity assumptions similar to or stronger than \ref{ass:reg}. Also, Assumption \ref{ass:bounded} has been used in the same references. 
Note that the above  assumptions do not compromise the distribution-free nature of conformal prediction. Given that obtaining distribution-free guarantees as in \eqref{full}  is impossible \cite{pmlr-v25-vovk12, 2019arXiv190304684F}, the adoption of regularity conditions to extend beyond mere marginal coverage guarantees is a well-established path. 

\subsection{Infinite data}
At the heart of our analysis, we have Proposition \ref{lem1} that connects the pinball loss to the MSCE of prediction sets. To motivate this Proposition, let us look at the following prediction set created by the true quantile function, $q_{1-\alpha}(S|X=x)$,
$$
    C_{\rm opt}(x)=\{y\in \mathcal{Y} \mid S(X, Y)\leq q_{1-\alpha}(S|X=x)\}.
$$
This prediction set is optimal as it guarantees full conditional coverage \eqref{full}. 
The true quantile function can also be expressed as the minimizer of the pinball loss; i.e. defining the function $q_{1-\alpha}(x) := q_{1-\alpha}(S|X=x) $, we have
\begin{align}\label{quantilemin}
q_{1-\alpha}(\cdot) \in
\underset{f:\mathcal{X}\rightarrow\R}{\text{argmin}}\: \E_{(X, S) \sim \mathcal{D}}\ell_\alpha(f(X), S).
\end{align}
This suggests that minimizing pinball loss can potentially lead to prediction sets with a better conditional coverage behaviour -- an intuition used in prior work, such as \cite{jung2023batch, gibbs2023conformal},  to obtain conditional guarantees using the pinball loss.  In the following proposition, we formalize this intuition, by bounding the MSCE of prediction sets constructed by an arbitrary function $g(x) : \mathcal{X \rightarrow \R}$. 

\begin{propo}\label{lem1}
Under assumption \ref{ass:reg}, for every function $g(x) : \mathcal{X \rightarrow \R}$, we have
\begin{align*}
    \text{MSCE}(C_g)\leq 
    &2L \E\left[\ell_{\alpha}(g(X), S) - \ell_{\alpha}(q_{1-\alpha}(X), S) \right],
\end{align*}
$$
\text{where:} \quad C_g(x) = \{y \in \mathcal{Y} | S(x, y) \leq g(x)\}.
$$
\end{propo}
In light of Proposition \ref{lem1}, we proceed with analyzing the coverage of the prediction sets created by PLCP. In the first step, we look at the case that we have infinitely many data, and the function class $\mathcal{H}=\Delta^\mathcal{X}_m$. Here $\Delta^\mathcal{X}_m$ denotes all the possible functions from $\mathcal{X}$ to simplex $\Delta_m$. In this case we show that the performance of function $h^\infty(x)$ defined in \eqref{secondary}, in terms of expected pinball loss,  is different from the optimal quantile function $q_{1-\alpha}(S|X=x)$ by an additive error $O(1/\sqrt{m})$. 
\begin{theorem}\label{thm1}
    For any distribution $\mathcal{D}$ and number of groups $m$, we have the following bound,
\begin{align}\label{thm1state}
     &\bigg|\underset{h \in \Delta^\mathcal{X}_m, \bq \in \mathbb{R}^m}{\min} \underset{(X, S) \sim \mathcal{D}}{\mathbb{E}}\left[\sum_{i=1}^m h^i(X)\ell_\alpha(q_i, S)\right]\\ \nonumber
     &\quad\quad - \mathbb{E}\left[\ell_{\alpha}(q_{1-\alpha}(X), S)\right]\bigg| \leq 2\sqrt{\frac{{\rm{var}}(q_{1-\alpha}(X))}{m-1}},
\end{align}
where ${\rm var}(\cdot)$ denotes the usual variance.
\end{theorem} 
An immediate implication of Theorem \ref{thm1} is the MSCE property of the prediction sets of PLCP. 
Recall the definitions of $h^\infty, \bq^\infty$ as in \eqref{secondary}. By using the results of Proposition \ref{lem1} and Theorem \ref{thm1}, we obtain the following corollary. 
\begin{corollary}\label{cor1}
Under Assumption~\ref{ass:reg}, the prediction sets $C_\infty=\{y\in\mathcal{Y}|S(x, y)\leq \bq^\infty_{i \sim h^\infty(x)}\}$,  satisfies the following conditional coverage guarantee 
\begin{align*}
    {\rm{MSCE}}( C_\infty)
    \leq 4L \sqrt{\frac{{\rm var}(q_{1-\alpha}(X))}{m-1}},
\end{align*}
where ${\rm var}(\cdot)$ denotes the variance.
\end{corollary}
Corollary \ref{cor1} indicates the effect of the number of groups, $m$, on the MSCE in the infinite-sample (population) regime. For further discussion on this result, please see Remarks \ref{remark1}, \ref{remark2} in the appendix.
\subsection{Finite data}
Next, we turn into the finite-size setting and analyze the performance of prediction sets provided by PLCP. Naturally, we should expect that the complexity of the function class $\mathcal{H}$ used by PLCP to play a role.   In this context, the following statement  captures the complexity of a function class using the well-known notion of covering number.
\begin{definition}
    ($\varepsilon$-net). Let $(T, d)$ be a metric space. Consider a subset $K \subset T$ and let $\varepsilon>0$. A subset $\mathcal{N} \subseteq K$ is called an $\varepsilon$-net of $K$ if every point in $K$ is within distance $\varepsilon$ of some point of $\mathcal{N}$, i.e. $\forall x \in K\quad  \exists y \in \mathcal{N}: d\left(x, y\right) \leq \varepsilon$.
\end{definition}
\begin{definition}
    (Covering numbers). The smallest possible cardinality of an $\varepsilon$-net of $K$ is called the covering number of $K$ and is denoted $\mathcal{N}(K, d, \varepsilon)$. 
\end{definition}
\begin{propo}
The following function $d:\Delta^\mathcal{X}_m \times\Delta^\mathcal{X}_m \rightarrow \R^+\cup\{0\}$ is a metric over $\Delta^\mathcal{X}_m$:
\begin{align*}
    d(h_1, h_2) = \sup_{x\in \mathcal{X}} \sum_{i=1}^{m} \left|h_1^i(x)-h_2^i(x)\right|,
\end{align*}
\end{propo}  
We will see that $\mathcal{N}(\mathcal{H}, d, \frac{1}{n})$ will play an important role in our finite sample analysis. Assuming that a class of functions has ``bounded complexity'', we would expect a trade-off between $m$, the number of groups, and $n$, the number of i.i.d. samples. On the one hand, corollary~\ref{cor2} suggests that increasing $m$ should benefit the conditional coverage of PLCP; and on the other hand, increasing $m$ in the finite sample regime would lead to a smaller number of samples per each group, that can hurt the coverage property of PLCP prediction sets. The next Theorem  characterizes this trade-off precisely. Before presenting the theorem we need to  define the approximation gap $\lambda_\mathcal{H}$ of a class $\mathcal{H}$ as
\begin{align}\label{gap}
    \lambda_\mathcal{H} = & \underset{h \in \mathcal{H}, \bq \in \mathbb{R}^m}{\text{min}} \underset{(X, S) \sim \mathcal{D}}{\mathbb{E}}\left[\sum_{i=1}^n h^i(X)\ell_\alpha(q_i, S)\right]\\ \nonumber
    &- \underset{h \in \Delta^\mathcal{X}_m, \bq \in \mathbb{R}^m}{\text{min}} \underset{(X, S) \sim \mathcal{D}}{\mathbb{E}}\left[\sum_{i=1}^m h^i(X)\ell_\alpha(q_i, S)\right].
\end{align}

In words, the approximation gap $\lambda_\mathcal{H}$ is the error we suffer due to restricting to the function class $\mathcal{H}$. Essentially, when $\lambda_\mathcal{H}=0$ (a.k.a. the minimizer is achieved by the class $\mathcal{H}$) then this becomes an analogous to the ``realizable case`` terminology that exists in the learning theory literature.

\begin{theorem}\label{thm2}
Under assumptions \ref{ass:iid} and \ref{ass:bounded} we have, with probability $1-\delta$,
\begin{align*}
    &\left|\underset{(X, S) \sim \mathcal{D}}{\mathbb{E}}\left[\sum_{i=1}^m {h^*}^i(X)\ell_\alpha({q^*_i}, S)\right] - 
    \mathbb{E}\left[l_{\alpha}(q_{1-\alpha}(X), S)\right]\right| \\
    &\leq 10 \sqrt{\frac{\ln\left( \frac{2}{\delta}\right) + \ln\left( \mathcal{N}(\mathcal{V}, ||.||_\infty, \epsilon_2)\right) + \ln \left(\mathcal{N}(\mathcal{H}, d, \frac{1}{n})\right)}{n}}\\
    &+ 2\sqrt{\frac{\text{var}(q_{1-\alpha}(X))}{m}} +\lambda_\mathcal{H} ,
\end{align*}
where $h^*, \bq^*$ are defined in algorithm \ref{PLCP}.
\end{theorem}
Applying Proposition \ref{lem1} to Theorem \ref{thm2} leads to the following corollary, which is our main theoretical result on the finite-size behavior of the predictions sets obtained by~PLCP.
\begin{corollary}\label{cor2}
    Under assumptions \ref{ass:iid}, \ref{ass:reg}, and \ref{ass:bounded} we have, with probability $1-\delta$,
    \begin{align*}
        {\rm MSCE}(C^*) \leq & 
        20L \sqrt{\frac{\ln\left( \frac{2}{\delta}\right) + m\ln\left(n\right) + \ln \left(\mathcal{N}(\mathcal{H}, d, \frac{1}{n})\right)}{n}} \\
        &+ 4L\sqrt{\frac{\text{var}(q_{1-\alpha}(X))}{m}} + 2L\lambda_\mathcal{H},
    \end{align*}
where $C^*$ is the prediction sets constructed by PLCP.
\end{corollary}
For further discussion, see Remark \ref{remark3} in the appendix.

\subsection{Fallback Coverage Guarantees}
The MSCE bounds provided in Corollaries \ref{cor1}, \ref{cor2}, can be simply translated to fallback coverage guarantees. 

Given prediction sets $C(x), x \in \mathcal{X}$, assume that 
$$\mathrm{MSCE}:= E\left[(\operatorname{Cov}(X)-(1-\alpha))^2\right] \leq p,$$
for some $p \geq 0$.
Using Jensen's inequality, we have
$$
1-\alpha-\sqrt{p} \leq E[\operatorname{Cov}(X)]=\Pr(Y \in C(X)) \leq 1-\alpha+\sqrt{p}
$$
This means any bound on MSCE gives a bound on the \emph{marginal} coverage of the prediction sets.

Furthermore, for \emph{any set} $A \subseteq \mathcal{X}$ such that $P[A] \geq \delta$, we have,
\begin{align*}
&1-\alpha-\sqrt{\frac{p}{\delta}} \leq E[\operatorname{Cov}(X) \mid X \in A]\\
&=\Pr(Y \in C(X) \mid X \in A) \leq 1-\alpha+\sqrt{\frac{p}{\delta}}
\end{align*}
This indicates that a bound on MSCE can be translated to \emph{conditional} coverage guarantees. 

We can now use the results of Corollaries \ref{cor1}, \ref{cor2} to obtain specific coverage guarantees for PCLP. These fallback coverage guarantees are provided in Corollary \ref{Fallback} in the appendix.

\section{Experimental Results}\label{exp}
We have conducted extensive numerical experiments to evaluate the performance of PLCP in terms of two key metrics: the coverage probability and the average length of the prediction intervals. These metrics are evaluated with respect to specific groups within the test data. For a given group of interest \(G\), conditional coverage and length are examined as quantities: $\Pr[Y \in C(X) \mid X \in G] $ and $\E[\text{length}(C(X)) \mid X \in G]$, where \(C(X)\) denotes the prediction set for a covariate point \(X\). The group \(G\) may represent a subpopulation within the data or a set of covariates that have undergone shifts relative to the calibration set. These shifts include features that are both Out of Distribution (OOD) and In Distribution (ID) w.r.t. the training set utilized for the predictor. We compare PLCP against three baselines: (i) the Split Conformal method \cite{Papadopoulos2002InductiveCM, Lei2016DistributionFreePI}; (ii) the BatchGCP method \cite{jung2023batch}; and (iii) the Conditional Calibration method \cite{gibbs2023conformal}. In all experiments, we set the miscoverage rate, $\alpha$, to 0.1. For a discussion on how to tune $m$, the number of regions in PLCP, see Remark \ref{pickm} in the~Appendix.

We have provided further experimental results in Appendix~\ref{add_exp} by comparing PLCP with several other baselines such as the method of Conformalized Quantile Regression (CQR) developed in \cite{romano2019conformalized} as well as the LocalCP method \cite{hore2023conformal}. These baselines are selected due to their emphasis on calibrated prediction sets that cater to conditional coverage. In Appendix \ref{add_exp}, we also look at two metrics recommended by \cite{feldman2021improving}\cite{feldman2021improving}, beyond coverage and set size plots. 

\subsection{Comparison with Split Conformal}\label{HS}
\begin{figure}[ht]
\centering
\includegraphics[width=0.49\columnwidth]{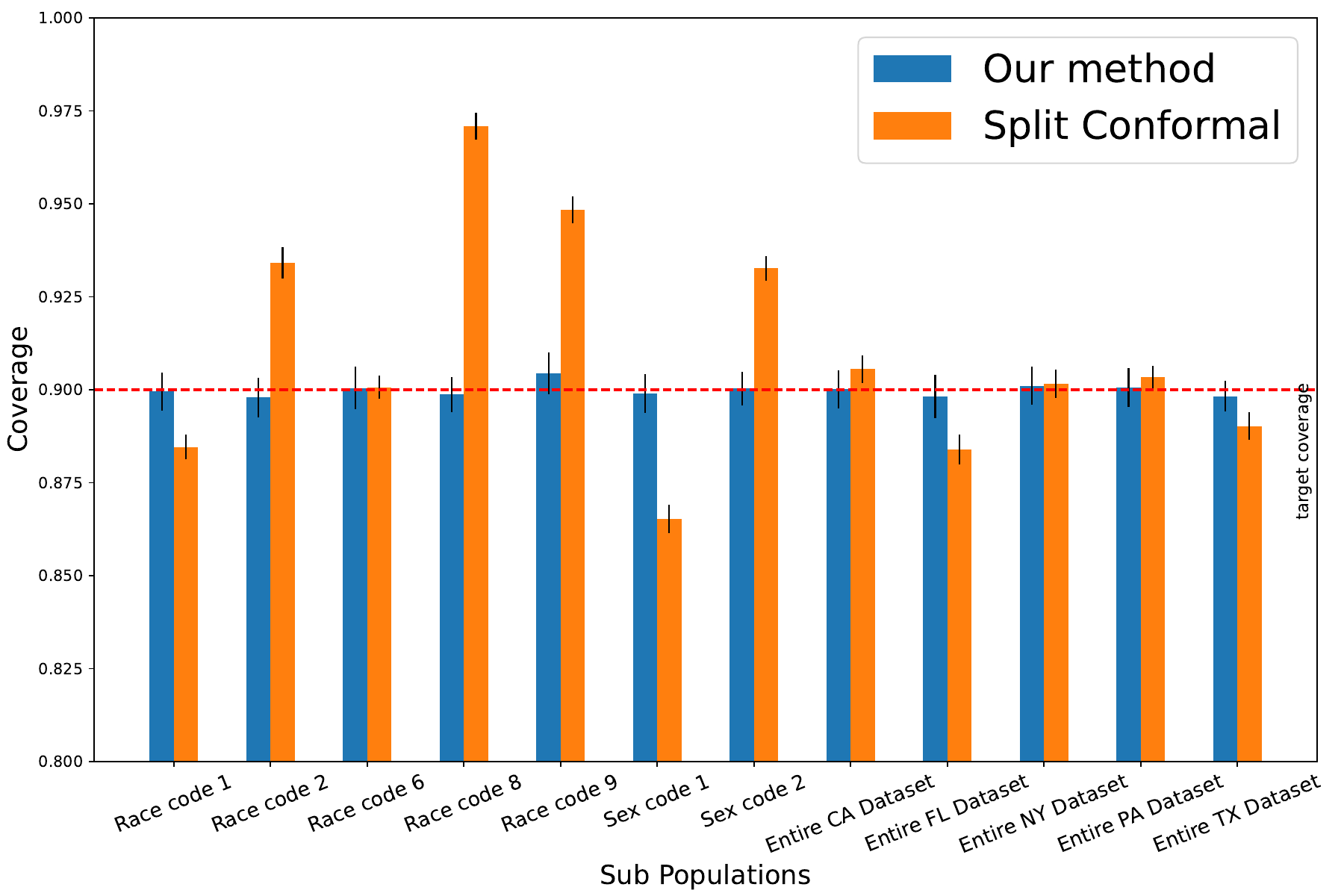}
\hfill
\includegraphics[width=0.49\columnwidth]{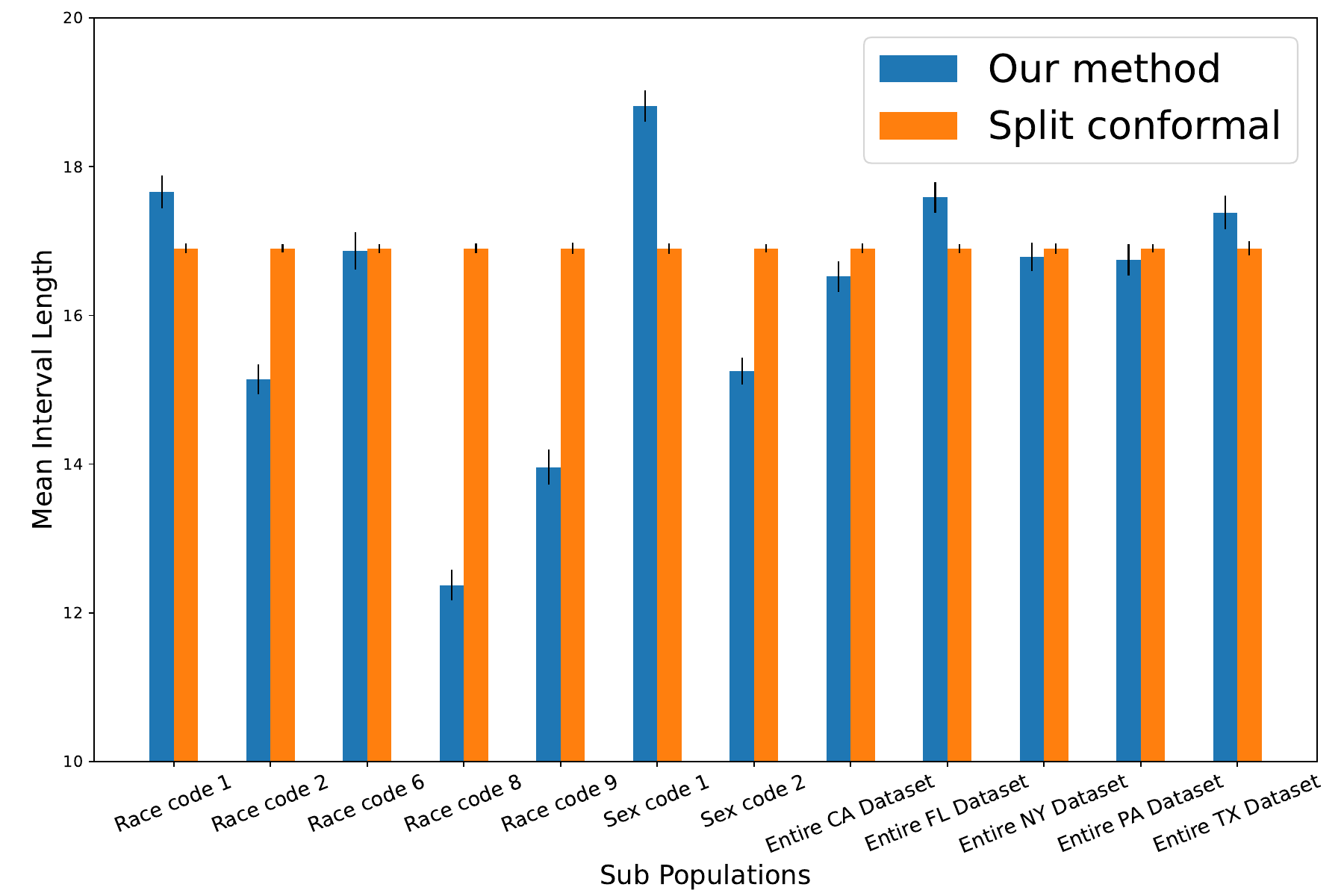}

\includegraphics[width=0.4951\columnwidth]{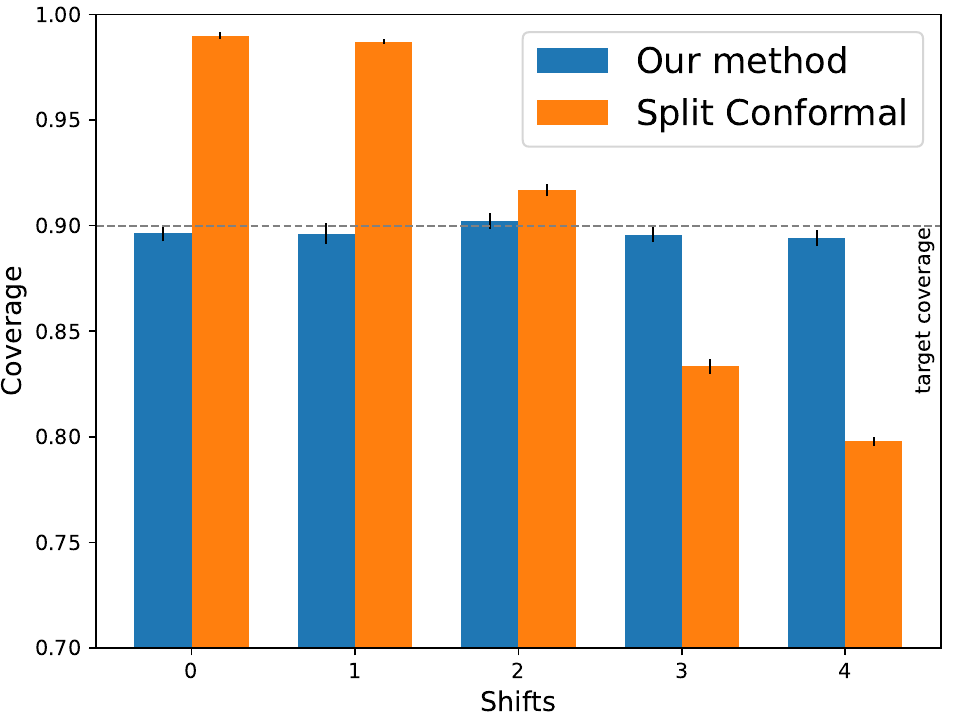}
\hfill
\includegraphics[width=0.4951\columnwidth]{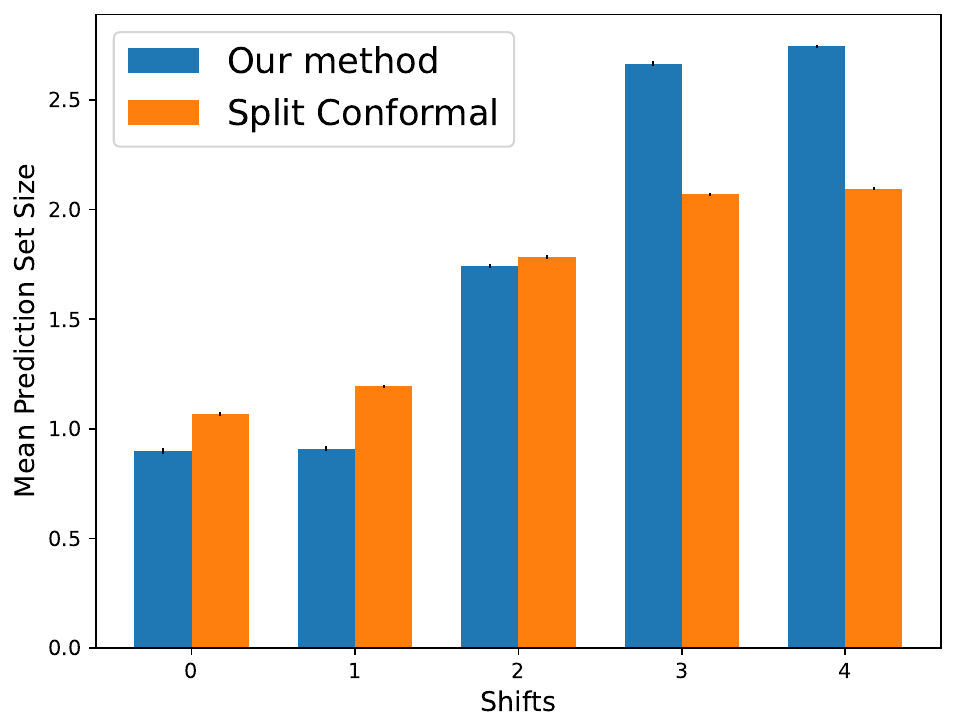}
\caption{Left-hand-side plots show coverage  and right-hand-side plots show mean prediction set size. Row 1: US Census Data; Row 2: MNIST with Gaussian~Blur.}
\label{fig:combined_experiments1}
\end{figure}
Here we compare PLCP with the Split Conformal method. While there are other methods with marginal guarantees, such as Full Conformal Prediction \cite{10.5555/1062391} and Jackknife+ \cite{barber2021predictive}, our choice of Split Conformal is driven by its  implementation efficiency and the fact that its conditional coverage performance is similar to those alternatives. The experiments are conducted in two scenarios depending on calibration data being In Distribution (ID) or Out of Distribution (OOD) relative to the training data.



\textbf{2018 US Census Data (In-Distribution).}
We study the 2018 US Census Data from the Folktables library \cite{ding2021retiring} for income prediction, a dataset rich in demographic details like gender, race, and age. Focusing on the five most populated US states (CA, FL, NY, PA, TX), we aim to model diverse demographic and geographical subpopulations. The objective is to assess PLCP's performance in scenarios where the train, calibration, and test sets are In Distribution (ID).
Data are divided into three segments: 60\% for training, 20\% for calibration, and 20\% for testing. We use a linear regression model as the predictor $f(x)$, with conformity measured by $S(x,y) = |f(x)-y|$. PLCP is implemented using a two-layer ReLU neural network (200 and 100 neurons). Performance evaluation across various racial, gender, and state-wise subpopulations is shown in Figure \ref{fig:combined_experiments1}. PLCP achieves near-perfect coverage across all subpopulations, which showcases PLCP's ability to learn and leverage features that are informative w.r.t. the uncertainty inherent in the predictor $f(x)$.

\textbf{MNIST Data (Out-of-Distribution).}
We divide the MNIST dataset into 35,000 training images and 25,000 for calibration/testing. The calibration/test data is further split into five subgroups each subjected to different Gaussian blur levels. Such added blurriness creates an OOD scenario compared to the training data. These 25,000 blurred images are then randomly divided into a 15,000-image calibration set and a 10,000-image test set. 
The conformity score for each input image $x$ is given as $S(x, y) = 1 - \pi^y(x)$, where $\{\pi^i(x)\}^{10}_{i=1}$ denote the softmax probabilities of class membership output by the trained predictive model. PLCP is implemented with \( m = 8 \) to denote eight distinct groups, employing a Convolutional Neural Network (CNN) architecture with three convolution layers and two feed-forward layers. The results are given in Figure \ref{fig:combined_experiments1}. PLCP demonstrates near-perfect coverage across all blur levels, effectively adjusting to the OOD conditions.

\subsection{Comparison with Methods that Provide Conditional Guarantees with Respect to Predefined Structures}
\textbf{Group-Conditional Methods.}
\begin{figure}[t!]
\centering
\includegraphics[width=0.49\columnwidth]{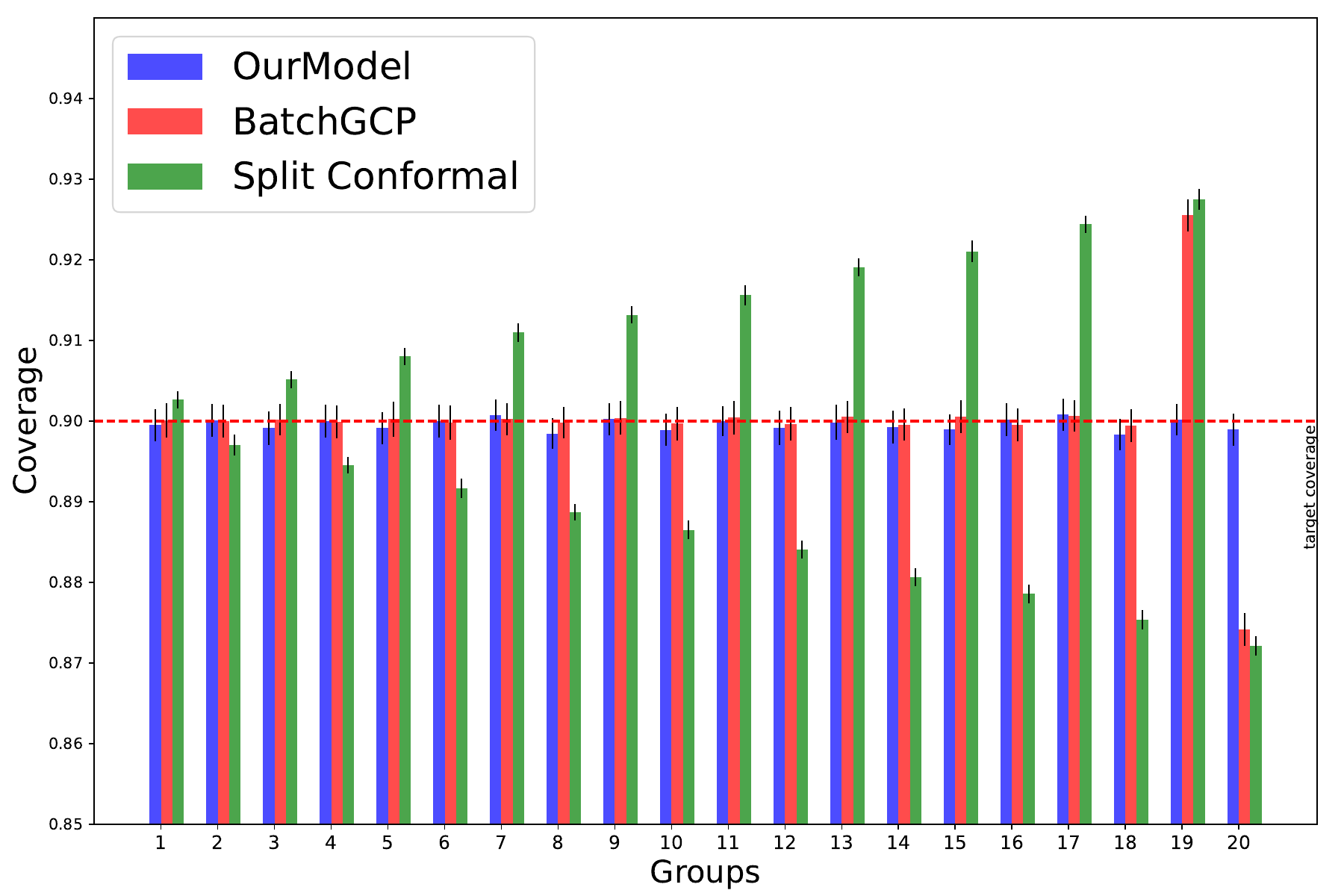}
\hfill
\includegraphics[width=0.49\columnwidth]{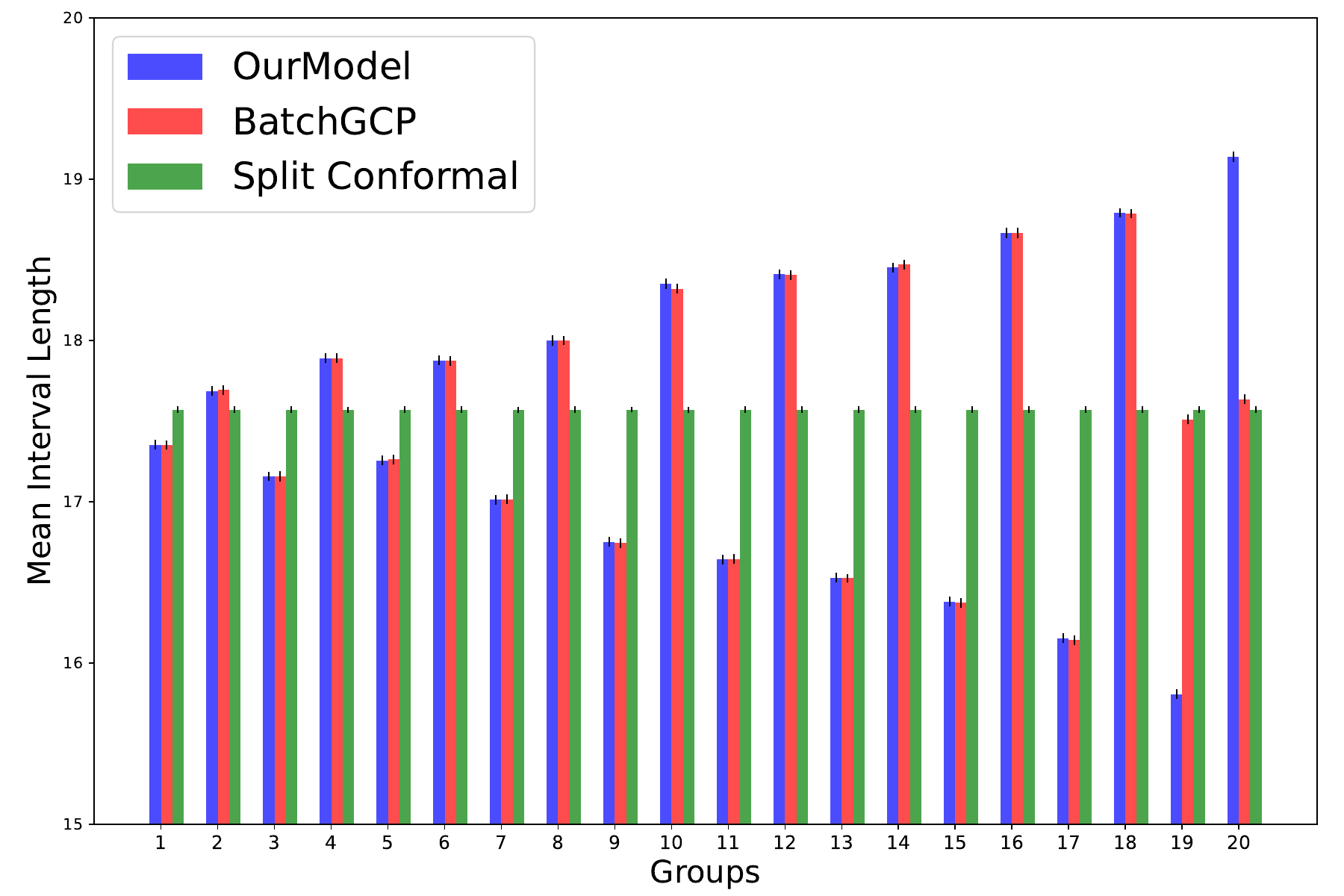}

\includegraphics[width=0.4951\columnwidth]{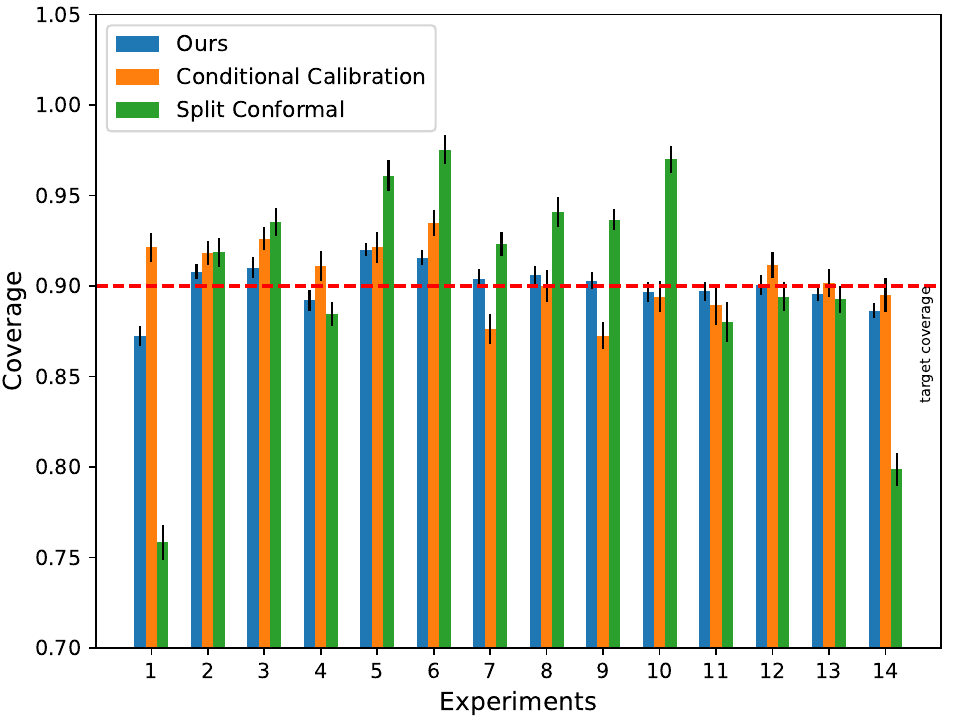}
\hfill
\includegraphics[width=0.4951\columnwidth]{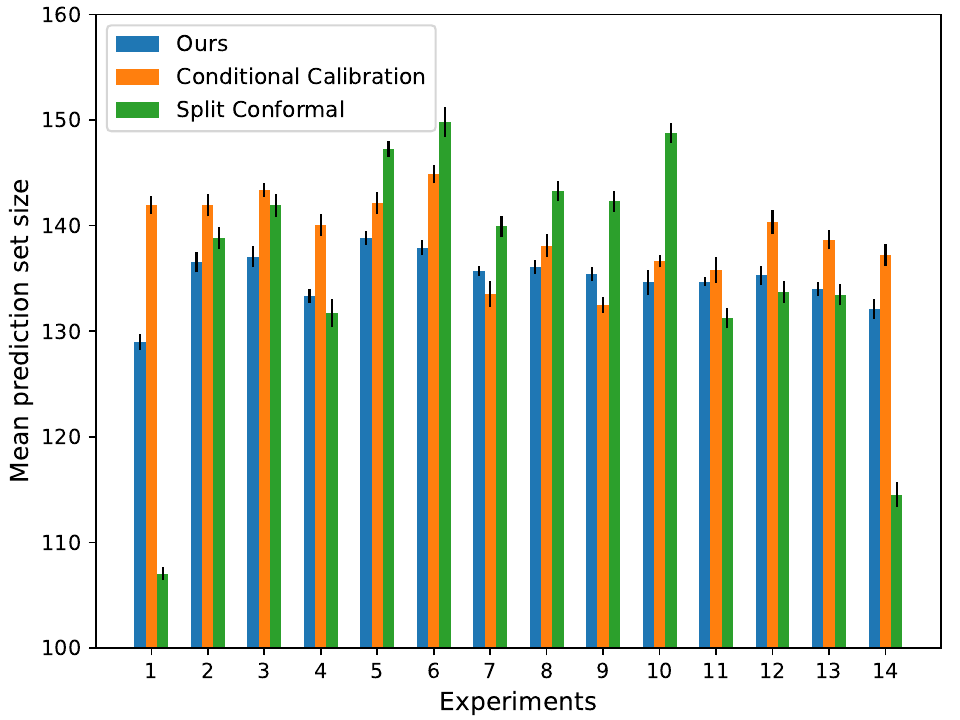}

\caption{Left-hand-side plots show coverage  and right-hand-side plots show mean prediction set size. Row 1: Synthetic Regression Task; Row 2: RxRx1 WILDS Dataset. \vspace{0cm}}
\label{fig:combined_experiments2}
\end{figure}
In this section, we used a synthetic regression task, proposed originally in \cite{jung2023batch}, to compare PLCP with BatchGCP.  The data distribution is given as follows:  The covariate $ X = (X_1, \cdots, X_{100})$  is a vector in \( \mathbb{R}^{100} \). The values in the first ten coordinates of \( X \) are independent uniform binary variables, and the remaining 90 coordinates are i.i.d. Gaussian variables with mean of zero and  variance of \( \sigma_x \). The label \( y \) is then generated as follows:
$y = \langle \theta, X \rangle + \epsilon_X$ where $\epsilon_X \sim \mathcal{N} \left(0, \sigma_x^2 + \sum_{i=1}^{d}X_i \cdot \sigma_i^2 \right)$,
and \( \sigma_i^2 = i \). We generate from this  distribution 150K training samples (to train the regression model to predict the label), 50K calibration data points, and 50K test data points. We evaluate all algorithms over 100 independent trials and report  average performance.

We define 20 overlapping groups based on the first ten binary components of \( X \). Specifically, for each \( i \) in the range 1 to 10, Group \( 2i-1 \) corresponds to \( X_i = 0 \) and Group \( 2i \) to \( X_i = 1 \). BatchGCP is implemented given the knowledge of the first 18 groups (associated with the first nine binary digits), however, no information about the group structures or the data distribution is provided to PLCP. We ran PLCP with \( m=25 \) (25 groups), using a linear classifier (as \(\mathcal{H}\)).

As seen in Figure \ref{fig:combined_experiments2}, the Split Conformal method results in over-coverage for groups with lower variance and under-coverage for those with higher variance. Notably, PLCP achieves coverage performance comparable to BatchGCP for the first 18 groups. However, for groups 19 and 20, whose information was not given to BatchGCP, PLCP exhibits superior coverage. As illustrated in the Mean Interval Length plot of Figure \ref{fig:combined_experiments2}, PLCP adaptively modifies the interval lengths for these two groups, recognizing the importance of \( X_{10} \), a feature learned from the data. This showcases the power of PLCP in learning   features/boundaries in covariate space that are informative about the uncertainty inside the predictions of the pre-trained model.

\textbf{Methods Based on a Family of Covariate Shifts.} 
Our last experiment is on the RxRx1 dataset \cite{taylor2019rxrx1} from the WILDS repository \cite{pmlr-v139-koh21a}, a dataset previously used in \cite{gibbs2023conformal}. This repository includes benchmark datasets particularly designed for evaluating performance under distribution shifts. The RxRx1 dataset comprises images of cells, where the task is to predict one of 1339 genetic treatments applied to these cells, spanning 51 distinct experiments. The inherent execution and environmental variations introduce covariate shifts among images from different experiments. One main challenge in this task is the lack of explicit features indicating experiment origin, which complicates the direct use of methods that rely on predefined groups or covariate shifts for coverage.  We compare PLCP with the Conditional Calibration algorithm \cite{gibbs2023conformal}. The Conditional Calibration approach bifurcates the calibration data: first, it uses $\ell_2$-regularized multinomial linear regression to estimate the likelihood of each image’s experimental origin, then applies these probabilities to define covariate shifts for the second part of the data. In contrast, PLCP inherently adapts to this task by directly learning features important for uncertainty quantification using the entire calibration set. For this experiment, we ran PLCP using a CNN with a single convolution layer, with ReLU activation followed by a linear layer, configured with $m=20$ groups.

For genetic treatment prediction (the predictive model), we employ a ResNet50 architecture, $f(x)$, pre-trained on 37 experiments from the WILDS repository.
The remaining images from the 14 experiments are divided into calibration and test sets uniformly at random. The conformity scores are computed as follows: for each image $x$, let $\{f^i(x)\}^{1339}_{i=1}$ represent the weights assigned by $f(x)$ to the 1339 treatments. Temperature Scaling, followed by a softmax operation, is applied to derive the probability weights $\pi^i(x):=\exp (T f_i(x) ) /(\sum_j \exp (T f_j(x)))$, with $T$ as the temperature parameter. We let $S(x, y):=\sum_{i: \pi_i(x)>\pi_y} \pi_i(x)$. Figure \ref{fig:combined_experiments2} provides a comparative evaluation of PLCP, Conditional Calibration, and Split Conformal. PLCP exhibits a superior performance in terms of set size while matching the coverage performance of Conditional Calibration. This performance, combined with the principled approach of our method, showcases its advantage for applications where identifying the correct uncertainty structures from the covariates is not straightforward.

\section*{Acknowledgements} 
Shayan Kiyani, Hamed Hassani, and George J.\ Pappas are supported by the NSF Institute for CORE Emerging Methods in Data Science (EnCORE).

\section{Impact Statement}
In this paper, we focus on  developing a new algorithmic framework for Conformal Prediction which has immediate applications in areas such as healthcare.  We do not anticipate any negative societal impact.


\bibliography{bibliography}
\bibliographystyle{icml2024}

\newpage
\appendix
\onecolumn
\section{Additional Experiments} \label{add_exp}
In this section we add two more baselines to the experiment setups of the Section \ref{HS}. We already compared PLCP with Split Conformal prediction, a conventional tool to achieve marginal guarantee, and similar to PLCP, it does not need to know any further structures of the data. Here we also consider and report the performance of two other algorithms, which are known to show better conditional coverage behavior than the Split Conformal solution; and they also do not need to have any prior knowledge of the structure of the data. The two methods are  the method of Conformalized Quantile Regression (CQR) developed in \cite{romano2019conformalized} as well as the LocalCP method \cite{hore2023conformal}. In short, CQR combines conformal prediction with classical quantile regression and makes the interval length more adaptable across the input space. LocalCP uses a weighted version of Split Conformal prediction, where the weights come from a local similarity  kernel (e.g. a gaussian). These methods are selected due to their emphasis on calibrated prediction sets that cater to conditional coverage. Furthermore, to make a more refined comparison, in addition to coverage and length plots, we also look at two metrics recommended by \cite{feldman2021improving}; Namely, the correlation between the size of prediction intervals and the indicators of coverage, assessed using both the Hilbert-Schmidt Independence Criterion (HSIC) and the Pearson’s correlation coefficient. In the tables provided (table \ref{table_mnist}, \ref{table_census}), the numbers represent percentages of improvement in comparison to the Split Conformal solution. This presentation mirrors the approach taken in the paper \cite{feldman2021improving} where correlations are similarly reported. In the same paper, the authors argue that for a prediction set to adhere to the full conditional coverage criterion (see \eqref{full}), its interval length and its coverage indicator must exhibit independence. Therefore, the improvement percentages detailed in the subsequent tables can be understood as the percentage reduction in the relevant correlation metric when measured against the Split Conformal solution. Tables \ref{table_mnist}, \ref{table_census} and the plots in Figure~\ref{fig:combined_experiments1_reb} show the superior performance of PLCP compared to the baselines.

\begin{table}[h!]
\centering
\caption{Results for the 2018 US Census}
\begin{tabular}{lccc}
\toprule
& PLCP & CQR & LocalCP \\
\midrule
HSIC & 56.07 & 51.89 & 32.12 \\
Pearson & 64.06 & 59.21 & 37.96 \\
\bottomrule
\end{tabular}\label{table_mnist}
\end{table}

\begin{table}[h!]
\centering
\caption{Results for the MNIST dataset}
\begin{tabular}{lccc}
\toprule
& PLCP & CQR & LocalCP \\
\midrule
HSIC & 73.36 & 58.01 & 40.88 \\
Pearson & 86.08 & 75.19 & 47.51 \\
\bottomrule
\end{tabular}\label{table_census}
\end{table}

\section{Remarks}

\begin{remark}\label{remark1}
  Theorems \ref{thm1} and Corollary \ref{cor1} can be  interpreted through the lens of quantization, a fundamental concept in information theory and signal processing. By constraining the number of groups to $m$, the quantity $\bq^\infty_{i\sim h^\infty(x)}$ takes only $m$ distinct values. This scenario can be conceptualized as an $m$-level  quantization ($\log(m)$ bits) of the signal $q_{1-\alpha}(X)$. With this perspective, Theorem \ref{thm1} provides an upper bound on the quantization error for the optimal $m$-level approximation of the optimal quantile function $q_{1-\alpha}(\cdot)$. Specifically, the theorem can be restated as follows: $\log\left(\frac{4\text{var}(q_{1-\alpha}(X))}{\epsilon^2}\right)$ bits suffice to achieve a quantization error below $\epsilon$ for the signal $q_{1-\alpha}(X)$. 
\end{remark}

\begin{remark}\label{remark2}
   The authors propose to explore  a more comprehensive theory that bridges the gap between quantization theory and conformal prediction with conditional guarantees as a promising avenue for future research. Such a theory could potentially lead to the development of more nuanced, distribution-dependent impossibility results, thereby extending the findings of \cite{pmlr-v25-vovk12, 2019arXiv190304684F}.
\end{remark}

\begin{remark}\label{remark3}
While similarities can be drawn between the finite sample theory presented in our work and the classical  PAC-learning framework, there exist fundamental differences in both algorithm design and analysis, making our analysis more complex. In the context of supervised classification, the learner is provided with covariates and their corresponding labels, and the challenge lies in developing a classifier that generalizes effectively to unseen data. In a similar manner, the task in conformal prediction can be described as predicting the hypothetical label $q_{1-\alpha}(S|X=x)$ for every covariate point $x$. However, examining the calibration data after applying the conformity measure, i.e. $\{(X_i, S_i)\}_{i=1}^n$, reveals that the learner has access to covariates and noisy approximations, $\{S_i\}_{i=1}^n$, of the hypothetical labels $\{q_{1-\alpha}(S|X=x)\}_{i=1}^n$. Therefore, the learner's objective is to devise a classifier that not only generalizes well to new data but also effectively navigates these noisy labels. This challenge is further compounded when considering that the random variable $S$ do not represent an unbiased estimation of $q_{1-\alpha}(S|X=x)$. Often in practical scenarios, $S$ tends to concentrate around the mean of its distribution, whereas $q_{1-\alpha}(S|X=x)$ corresponds to a point in the upper tail of $S$'s probability density function. This discrepancy indicates that our methodology and theoretical analysis are tackling a scenario of greater sophistication compared to the standard supervised classification problem, encompassing not just prediction accuracy but also the subtleties of dealing with biased and noisy label information.
 
\end{remark}

\begin{remark}\label{candes}
    The optimization problem considered in \citet{gibbs2023conformal}, similar to ours, is based on minimizing a pinball loss-oriented objective over a class of functions. However, these two approaches are fundamentally different. The role of the function class in the method of \citet{gibbs2023conformal}, as mentioned by the authors, is to learn\textbackslash{}approximate the quantile value function ($q_{1-\alpha}(x)$ in our notation). In other words, they address a \emph{quantile regression} problem over calibration data. On the contrary, the function class in our method is used to learn a partitioning\textbackslash{}boundary in the covariate space, effectively tackling a \emph{classification} problem over the calibration data. Although this boundary is related to the conditional quantile function ($q_{1-\alpha}(x)$), the relationship is considerably more complex. To illustrate, consider two scenarios in which the conditional quantile functions differ by a constant real number $c$ (a situation that could arise if two datasets experience a constant label shift relative to each other). Then, the quantile regression function produced by the method of \citet{gibbs2023conformal} in these scenarios, should ideally differ by the same constant $c$. Conversely, the output of the function obtained by our method should ideally remain unchanged across both scenarios. This distinction stems from the fact that the model obtained by our method aims to identify regions in the covariate space whose members exhibit similar uncertainty levels (with respect to the predictions of the underlying pretrained predictive model), rather than mirroring the exact values of the conditional quantile function ($q_{1-\alpha}(x)$).
\end{remark}
\begin{remark}\label{pickm}
    In practical applications, the hyperparameter $m$ is optimized through cross-validation. Our approach to tune $m$, as informed by the theoretical insights in Section \ref{thm}, leverages the observed bell-shaped relationship between $m$ and accuracy (measured by MSCE). Starting with $m=1$, we note that the accuracy initially improves with increasing $m$, but after a certain point it begins to decline. To identify the optimal $m$, we employ the doubling trick: setting aside 20 percent of the calibration data for validation, we increment $m$ from a small value, evaluate PLCP on the validation set, and continue doubling $m$ until the validation metric worsens. We then fine-tune by bisecting between the last two $m$ values. Once the optimal $m$ is determined, we re-run PLCP on the entire calibration set using this optimized $m$ value. 
\end{remark}
\section{Additional Corollaries}
\begin{corollary}\label{Fallback}
    \textbf{Infinite sample}: Under Assumption~\ref{ass:reg}, the following fallback coverage guarantees hold,

    (a) Marginal validity:
    \begin{align*}
    \bigg|\Pr(Y \in C_\infty(X)) - (1-\alpha)\bigg|\leq O\left(m^{-\frac{1}{4}}\right)
    \end{align*}
    (b) Conditional validity: For any set $A \subseteq \mathcal{X}$ such that $\Pr(A) \geq \gamma$
    \begin{align*}
    \bigg|\Pr(Y \in C_\infty(X) \mid X \in A)-(1-\alpha)\bigg|\leq O\left(m^{-\frac{1}{4}} \gamma^{-\frac{1}{2}}\right)
    \end{align*}

    \textbf{Finite sample}: Under assumptions \ref{ass:iid}, \ref{ass:reg}, and \ref{ass:bounded} we have,
    
    (a) Marginal validity: With probability $1-\delta$,
    \begin{align*}
    \bigg|\Pr(Y \in C(X)) - (1-\alpha)\bigg| \leq O\left(\sqrt{m^{-\frac{1}{2}}+\sqrt{\frac{\ln\left( \frac{2}{\delta}\right)+m \ln (n)+\ln\left(\mathcal{N}\left(\mathcal{H}, d, \frac{1}{n}\right)\right)}{n}}+\lambda_{\mathcal{H}}}\right)
    \end{align*}
    
(b) Conditional validity: For any set $A \subseteq \mathcal{X}$ such that $P[A] \geq \gamma$, with probability $1-\delta$,
\begin{align*}
\bigg|\Pr(Y \in C(X) \mid X \in A) - (1-\alpha)\bigg|\leq O\left(\gamma^{-\frac{1}{2}} \sqrt{m^{-\frac{1}{2}}+\sqrt{\frac{\ln\left( \frac{2}{\delta}\right)+m \ln (n)+\ln \left(\mathcal{N}\left(\mathcal{H}, d, \frac{1}{n}\right)\right)}{n}}+\lambda_{\mathcal{H}}}\right)
\end{align*}

\end{corollary}

\section{Additional Lemmas}
\begin{theorem}
(Chernoff Bound). Let $\left\{X_i\right\}_{i=1}^n$ be independent random variables bounded such that for each $i \in[n], X_i \in[0,1]$. Let $S_n=\sum_{i=1}^n X_i$ denote their sum. Then for all $\epsilon>0$,
$$
\Pr_{\left\{X_i\right\}_{i=1}^n}\left(\left|S_n-\mathbb{E}\left[S_n\right]\right| \geq \epsilon\right) \leq 2 \exp \left(-\frac{\epsilon^2}{n}\right)
$$
\end{theorem}

\begin{lemma}\label{LipLem} (Lipschitz Continuity of the Pinball Loss Function). The pinball loss function exhibits Lipschitz continuity with a constant of 1. This implies that for any four real numbers \( y_1, y_2, y_3, y_4 \in \mathbb{R} \), the following inequality holds true:
\[
\left|\ell_\alpha\left(y_1, y_2\right) - \ell_\alpha\left(y_3, y_4\right)\right| \leq \left|\left(y_1 - y_2\right) - \left(y_3 - y_4\right)\right|.
\]
\end{lemma}
\begin{proof}
Our objective is to establish that \( \ell_\alpha\left(y_1, y_2\right) - \ell_\alpha\left(y_3, y_4\right) \leq \left|\left(y_1 - y_2\right) - \left(y_3 - y_4\right)\right| \). The converse inequality can be derived analogously due to symmetry. We consider the following four scenarios:

\textbf{Scenario 1}: When \( y_1 \geq y_2 \) and \( y_3 \geq y_4 \), we have:
\[
\ell_\alpha\left(y_1, y_2\right) - \ell_\alpha\left(y_3, y_4\right) = \alpha\left(y_1 - y_2\right) - \alpha\left(y_3 - y_4\right) \leq \left|\left(y_1 - y_2\right) - \left(y_3 - y_4\right)\right|.
\]

\textbf{Scenario 2}: For \( y_1 < y_2 \) and \( y_3 < y_4 \), it follows that:
\[
\ell_\alpha\left(y_1, y_2\right) - \ell_\alpha\left(y_3, y_4\right) = (1 - \alpha)\left(y_2 - y_1\right) - (1 - \alpha)\left(y_4 - y_3\right) \leq \left|\left(y_1 - y_2\right) - \left(y_3 - y_4\right)\right|.
\]

\textbf{Scenario 3}: In the case where \( y_1 \geq y_2 \) but \( y_3 < y_4 \), the equation becomes:
\[
\begin{aligned}
\ell_\alpha\left(y_1, y_2\right) - \ell_\alpha\left(y_3, y_4\right) &= \alpha\left(y_1 - y_2\right) - (1 - \alpha)\left(y_4 - y_3\right) \\
&= \alpha\left(y_1 - y_2 - \left(y_3 - y_4\right)\right) + \left(y_3 - y_4\right) \leq \left|\left(y_1 - y_2\right) - \left(y_3 - y_4\right)\right|.
\end{aligned}
\]

\textbf{Scenario 4}: Lastly, when \( y_1 < y_2 \) and \( y_3 \geq y_4 \), we have:
\[
\begin{aligned}
\ell_\alpha\left(y_1, y_2\right) - \ell_\alpha\left(y_3, y_4\right) &= (1 - \alpha)\left(y_2 - y_1\right) - \alpha\left(y_3 - y_4\right) \\
&= (1 - \alpha)\left(y_2 - y_1 - \left(y_4 - y_3\right)\right) + \left(y_4 - y_3\right) \leq \left|\left(y_1 - y_2\right) - \left(y_3 - y_4\right)\right|.
\end{aligned}
\]
\end{proof}

\begin{lemma}\label{covering}
Let \( B_{\infty}^m \) denote the unit ball in \(\mathbb{R}^m\) under the sup norm, defined by \( B_{\infty}^m = \{ x \in \mathbb{R}^m : \| x \|_{\infty} \leq 1 \} \). For any positive real number \( \delta \), smaller than 2, the covering number \( N(B_{\infty}^m, \delta) \) of \( B_{\infty}^m \) can be bounded above by $\lceil \frac{2}{\delta} \rceil^m$.
\end{lemma}

\begin{proof}
Consider a grid of cubes in \(\mathbb{R}^m\), each of side length \( \delta \), where \( 0 < \delta < 2 \). These cubes are the Cartesian product of intervals of length \( \delta \), centered at grid points in \(\mathbb{R}^m\). A cube centered at a point \( y = (y_1, y_2, \ldots, y_m) \) is defined as the set \( \{ x \in \mathbb{R}^m : |x_i - y_i| \leq \delta/2, \forall i = 1, \ldots, m \} \). The grid points form a regular lattice in \(\mathbb{R}^m\), with each coordinate being an integer multiple of \( \delta \).

Any point $ x \in B_{\infty}^m $ will lie within a distance of $ \frac{\delta}{2} $ (in the sup norm) from some grid point. The unit ball under the sup norm can be inscribed within a cube of side length 2, centered at the origin. This larger cube can be covered by at most $ \lceil \frac{2}{\delta} \rceil $ intervals of length $ \delta $ along any axis. Therefore, in $ m $ dimensions, the total number of cubes required to cover $ B_{\infty}^m $ is bounded above by $ \lceil \frac{2}{\delta} \rceil^m $.
\end{proof}

\section{Proofs of Section \ref{thm}}\label{Proofs}
\begin{proof} [Proof of Proposition \ref{lem1}]
Starting from the right-hand side, we have
\begin{align}
    \E[\ell_{\alpha}(g(X), S) - \ell_{\alpha}(q_{1-\alpha}(X), S)] &= \mathbb{E}_X \E_{S|X}[\ell_{\alpha}(g(X), S) - \ell_{\alpha}(q_{1-\alpha}(X), S)] \\ \label{eq:4}
    &= \E_X[\gamma(X, g(X)) - \gamma(X, q_{1-\alpha}(X))],
\end{align}
where $\gamma(x,q) := \mathbb{E}_{S|X=x}[\ell_{\alpha}(q, S)]$.

Note that the derivative of $\gamma$ with respect to $q$ is
\begin{align} \label{derivative_of_gamma}
    \gamma^{'}(x,q) = \frac{d}{dq} \gamma(x,q) = \Pr[S < q \mid X = x] - (1 - \alpha),\quad \forall q \in \R.
\end{align}
This will allow us to connect the pinball loss to conditional coverage of the prediction sets. To do so, we first show that we can bound $\gamma(.,.)$ in the following way 
\begin{equation} \label{claim}
\gamma(x,q) - \gamma(x, q_{1-\alpha}(S \mid X = x)) \leq
\frac{\gamma^{'}(x,q)^2}{2L}
\end{equation}
To show \eqref{claim}, we assume that  $q \geq q_{1-\alpha}(S \mid X = x)$ and note that for the other case, i.e. $q < q_{1-\alpha}(S \mid X = x)$, the proof follows similarly. 

We can write
\begin{align*}
    \gamma(x,q) - \gamma(x, q_{1-\alpha}(S \mid X = x)) 
    &= \gamma(x,q) - \gamma\left(x,q - \frac{\gamma^{'}(x,q)}{L}\right) + \gamma\left(x,q - \frac{\gamma^{'}(x,q)}{L}\right) - \gamma(x,q_{1-\alpha}(S \mid X = x)) \\ 
    &\stackrel{(a)} {\geq}\gamma(x,q) - \gamma\left(x,q - \frac{\gamma^{'}(x,q)}{L}\right) \\ 
    & \stackrel{(b)} {\geq} \int_{q - \frac{\gamma^{'}(x,q)}{L}}^{q} \gamma^{'}(x,\tilde{q}) d\tilde{q} \\
    &\stackrel{(c)}{\geq} \int_{q - \frac{\gamma^{'}(x,q)}{L}}^{q} \left[\gamma^{'}(x,q) - L(q - \tilde{q})\right] d\tilde{q} \\
    &= \frac{\gamma^{'}(x,q)^2}{L} - \frac{\gamma^{'}(x,q)^2}{2L} \\
    &= \frac{\gamma^{'}(x,q)^2}{2L},
\end{align*}
where (a) follows from \eqref{quantilemin}, (b) is due to the fundamental theorem of calculus, and (c) follows from assumption \ref{ass:reg} which results in the $L$-Lipschitz continuity of $\gamma^{'}(x,q)$ in term of $q$.
This will conclude the proof of \eqref{claim}.

Continuing from \eqref{eq:4}, we have
\begin{align*}
    \mathbb{E}[\ell_{\alpha}(g(X), S) - \ell_{\alpha}(q_{1-\alpha}(X), S)] &\stackrel{\eqref{claim}}{\geq} \mathbb{E}_X\left[\frac{\gamma^{'}(X,g(X))^2}{2L}\right] \\
    &\stackrel{\eqref{derivative_of_gamma}}{=} \mathbb{E}_X\left[\frac{\left(\Pr[S < g(X) \mid X = x] - (1 - \alpha) \right)^2}{2L}\right] \\
    &= \frac{\text{MSCE}(C_g)}{2L},
\end{align*}
which concludes the proof of the proposition.
\end{proof}

\begin{proof}[Proof of Theorem~\ref{thm1}]
By using the Jensen inequality we have,
\begin{align*}
    \sum_{i=1}^m h^i(X)\ell_\alpha(q_i, S) \geq \ell_\alpha \left(\sum_{i=1}^m h^i(X)q_i, S \right) \quad \forall \: h \in \mathcal{X}^{\Delta_m},\:\forall\: \bq \in \mathbb{R}^m,
\end{align*}
which is a consequence of the convexity of the pinball loss. Taking an expectation and a minimum from both sides we have,
\begin{align*}
    \underset{h \in \Delta^\mathcal{X}_m, \bq \in \mathbb{R}^m}{\text{min}} \underset{(X, S) \sim \mathcal{D}}{\mathbb{E}}\left[\sum_{i=1}^m h^i(X)\ell_\alpha(q_i, S)\right] &\geq \underset{h \in \Delta^\mathcal{X}_m, \bq \in \mathbb{R}^m}{\text{min}} \underset{(X, S) \sim \mathcal{D}}{\mathbb{E}}\left[\ell_\alpha(\sum_{i=1}^m h^i(X)q_i, S)\right]\\
    &\geq\mathbb{E}\left[\ell_{\alpha}(q_{1-\alpha}(X), S)\right]
\end{align*}
where the second inequality comes from \eqref{quantilemin}. This means the quantity inside the absolute value on the  left-hand-side of the Theorem statement \eqref{thm1state} is positive.

 At a high level, the proof will proceed as follows: We first show that we can bound the left-hand-side quantity in \eqref{thm1state} by looking at an arbitrary partitioning of the space $\mathcal{X}$. Then, we further refine the bound by looking at a very specific partitioning of the space $\mathcal{X}$. 
Now, assume that $E = \{E_i\}_{i=1}^{m+1}$ is a partitioning on the set $\mathcal{X}$; i.e. 
$$
E=\{E_i\}_{i=1}^{m+1} \text{ such that } \bigcup_{i=1}^{m+1} E_i = \mathcal{X}, \text{and}\: E_i \cap E_j = \emptyset \quad \forall i \neq j,
$$
and $\boldsymbol{\tilde{q}}$ is an arbitrary vector in $\R^m$. We can write,
\begin{align} \nonumber
    \underset{h \in \Delta^\mathcal{X}_m, \bq \in \mathbb{R}^m}{\text{min}} \E_{(X, S)}\left[\sum_{i=1}^m h^i(X)\ell_\alpha(q_i, S)\right] &= 
    \underset{h \in \Delta^\mathcal{X}_m, \bq \in \mathbb{R}^m}{\text{min}} \E_X\left[ \mathbb{E}_{S|X} \left[\sum_{i=1}^m h^i(X)\ell_\alpha(q_i, S)\right]\right]
    \\
    \nonumber
    &=\underset{h \in \Delta^\mathcal{X}_m, \bq \in \mathbb{R}^m}{\text{min}} \E_X\left[\sum_{i=1}^m \mathbb{E}_{S|X} \left[ h^i(X)\ell_\alpha(q_i, S)\right]\right]
    \\
    \nonumber
    &=\underset{h \in \Delta^\mathcal{X}_m, \bq \in \mathbb{R}^m}{\text{min}} \E_X\left[\sum_{i=1}^m h^i(X)\mathbb{E}_{S|X} \left[\ell_\alpha(q_i, S)\right]\right]
    \\
    \nonumber
    &\stackrel{(a)}{=} \underset{\bq \in \mathbb{R}^m}{\text{min}} \E_X\left[\min_{i \in [m]} \mathbb{E}_{S|X} \left[\ell_\alpha(q_i, S)\right]\right]
    \\
    \label{partition-expectation}
    &\stackrel{(b)}{\leq} \sum_{i=1}^m \int_{x \in E_i} p(x) \mathbb{E}_{S|X=x} \left[ \ell_{\alpha}(\tilde{q}_i, S) \right] \, dx
\end{align}
Where (a) comes from the fact that the min over $h$ is taken over all the functions in $\Delta^\mathcal{X}_m$, and (b) is due to the fact that $E$ partitions the set $\mathcal{X}$.

A similar approach can be applied to reformulate $\mathbb{E}\left[\ell_{\alpha}(q_{1-\alpha}(X), S)\right]$ in terms of the partition $E$.
\begin{align}
    \nonumber
    \E_{(X, S)}\left[\ell_{\alpha}\left(q_{1-\alpha}(S|X), S\right)\right] &= \E_{X}\E_{S|X} \left[\ell_{\alpha}\left(q_{1-\alpha}(S|X), S\right)\right] 
    \\
    \label{optpart}
    &=\sum_{i=1}^m \int_{x \in E_i} p(x)\E_{S|X=x} \left[\ell_{\alpha}\left(q_{1-\alpha}(S|X=x), S\right)\right] \, dx
\end{align}
Now putting together \eqref{partition-expectation} and \eqref{optpart}, we can write,
\begin{align} 
    \nonumber
    \underset{h \in \Delta^\mathcal{X}_m, \bq \in \mathbb{R}^m}{\text{min}} \E_{(X, S)}&\left[\sum_{i=1}^m h^i(X)\ell_\alpha(q_i, S)\right]-\E_{(X, S)}\left[\ell_{\alpha}\left(q_{1-\alpha}(S|X), S\right)\right] 
    \\
    \nonumber
    &\leq  \sum_{i=1}^m \int_{x \in E_i} p(x)\big[\mathbb{E}_{S|X=x} \left[ \ell_{\alpha}\left(\tilde{q}_i, S\right)-\ell_{\alpha}\left(q_{1-\alpha}(S|X=x), S\right)\right]\big] \, dx
    \\ 
    \label{lip-part}
    &\leq \sum_{i=1}^m \int_{x \in E_i} p(x)\biggl| \, \tilde{q}_i-q_{1-\alpha}(S|X=x)  \, \biggr| \, dx
\end{align}
Here, the last inequality is due to Lemma \ref{LipLem} which addresses the Lipschitzness of the pinball loss. 

The relation \eqref{lip-part}  holds for any (arbitrary) partition $E$ and vector $\boldsymbol{\tilde{q}}$. Consequently, by using a carefully crafted  partition of the space $\mathcal{X}$, we will be able to obtain a tighter bound for our problem.  Let us fix the following partitioning of the set $\mathcal{X}$ and vector, keeping the same notation $E$ and $\boldsymbol{\tilde{q}}$.
Let $\eta$ be an arbitrary positive real number,
$$E_i = \left\{ x \in \mathcal{X} \,  \text{ such that: } \,  \;\left| \, q_{1-\alpha}(S|X=x)-\E[q_{1-\alpha}(S|X)] + \eta - \frac{(2i-1)\eta}{m-1} \,  \right| \leq \frac{\eta}{m-1} \right\} \quad \forall \; i \in [1, \cdots, m-1],$$ $$\Tilde{q}_i=\E[q_{1-\alpha}(S|X)] - \eta + \frac{(2i-1)\eta}{m-1}\quad \forall \; i \in [1, \cdots, m-1],$$
and
$$
E_m = \left\{ x \in \mathcal{X} \,  \text{ such that: } \, \;\biggl| \, q_{1-\alpha}(S|X=x)-\E[q_{1-\alpha}(S|X)] \, \biggr| \geq \eta \right\}, 
$$
$$ 
\Tilde{q}_m=\E[q_{1-\alpha}(S|X)].
$$
By definition, $\bigcup_{i=1}^{m}E_i$ constitutes a partition of $\mathcal{X}$, and $\boldsymbol{\tilde{q}}=(\Tilde{q}_1, \cdots, \Tilde{q}_m)$ is a vector in $\mathcal{R}^m$. Substituting back into the inequality \eqref{lip-part}, we obtain:
\begin{align} 
     \nonumber
     \underset{h \in \Delta^\mathcal{X}_m, \bq \in \mathbb{R}^m}{\text{min}} \E_{(X, S)}&\left[\sum_{i=1}^m h^i(X)\ell_\alpha(q_i, S)\right]-\E_{(X, S)}\left[\ell_{\alpha}\left(q_{1-\alpha}(S|X), S\right)\right]
     \\
     \nonumber
     &\leq \sum_{i=1}^m \int_{x \in E_i} p(x)\biggl|\, \Tilde{q}_i-q_{1-\alpha}(S|X=x) \, \biggr| \, dx
     \\ 
     \label{ineqlast}
     &\leq \frac{\eta}{m-1} +  \int_{x \in E_m} p(x)\biggl| \, q_{1-\alpha}(S|X=x)-\E[q_{1-\alpha}(S|X)] \, \biggr| \, dx
\end{align}
To finish the proof, it remains to bound the second term in \eqref{ineqlast}. We now show that for every $\eta > 0$ we have,
    \begin{equation} \label{clmvar}
    \int_{x \in E_m} p(x)\left| q_{1-\alpha}(S|X=x)-\E[q_{1-\alpha}(S|X)]\right| \, dx \leq \frac{\operatorname{Var}(q_{1-\alpha}(S|X))}{\eta}
    \end{equation}
To see the above inequality, we can write 
    \begin{align*}
        \operatorname{Var}(q_{1-\alpha}(S|X)) =& \int_{E_m} p(x) \left(q_{1-\alpha}(S|X=x)-\E[q_{1-\alpha}(S|X)\right)^2 dx + \int_{{E_m}^c} p(x) \left(q_{1-\alpha}(S|X=x)-\E[q_{1-\alpha}(S|X)\right)^2 dx\\
        &\geq \int_{E_m} p(x) \left(q_{1-\alpha}(S|X=x)-\E[q_{1-\alpha}(S|X)\right)^2dx\\
        &\geq \eta \int_{E_m} p(x) \left|q_{1-\alpha}(S|X=x)-\E[q_{1-\alpha}(S|X)\right|dx
    \end{align*}
Rearranging the last inequality proves \eqref{clmvar}. 

Finally, by plugging  \eqref{clmvar} into \eqref{ineqlast}, we obtain
\begin{align*}
    \underset{h \in \Delta^\mathcal{X}_m, \bq \in \mathbb{R}^m}{\text{min}} \E_{(X, S)}&\left[\sum_{i=1}^m h^i(X)\ell_\alpha(q_i, S)\right]-\E_{(X, S)}\left[\ell_{\alpha}\left(q_{1-\alpha}(S|X), S\right)\right]\\
    &\leq \frac{\eta}{m-1} +  \frac{\operatorname{Var}(q_{1-\alpha}(S|X))}{\eta},
\end{align*}
and choosing $\eta = \sqrt{(m-1)\operatorname{Var}(q_{1-\alpha}(S|X))}$ gives us
$$
\underset{h \in \Delta^\mathcal{X}_m, \bq \in \mathbb{R}^m}{\text{min}} \E_{(X, S)}\left[\sum_{i=1}^m h^i(X)\ell_\alpha(q_i, S)\right]-\E_{(X, S)}\left[\ell_{\alpha}\left(q_{1-\alpha}(S|X), S\right)\right]\leq 2\sqrt{\frac{\operatorname{Var}(q_{1-\alpha}(S|X))}{m-1}}
$$.
\end{proof}

\begin{proof}[Proof of Theorem \ref{thm2}]
We start by defining the following random variable
\begin{equation}
    Z_{h,\bq}=\frac{1}{n} \sum_{j=1}^n \sum_{i=1}^m h^i(X_j)\ell_\alpha(q_i, S_j) -\underset{(X, S) \sim \mathcal{D}}{\mathbb{E}}\left[\sum_{i=1}^m h^i(X)\ell_\alpha(q_i, S)\right].
\end{equation}
Utilizing Lemma \ref{LipLem} on Lipschitzness of pinball loss and assumption \ref{ass:bounded} we have with probability at least $1-\delta$ that,
\begin{align}\label{cher}
    \left|Z_{h,\bq}\right| \leq  \sqrt{\frac{\ln \left(\frac{2}{\delta}\right)}{n}}.
\end{align}
The rest of the proof is dedicated to show that a similar bound to \eqref{cher} holds with high probability simultaneously for all the possible values of $q$ and $h$. Recall the definitions of $\bq^*$ \eqref{final} and $\bq^\infty$ \eqref{secondary}. Here one key observation is, since the random variable $S$ only takes values in $[0, 1]$, both $\bq^\infty$ and $\bq^*$ should have all their entries between $0$ and $1$. This can be shown for $\bq^*$ by looking at the first order condition of \eqref{final} (similarly for $\bq^\infty$ by looking at the \eqref{secondary}), for every $i \in [1, \cdots, m]$,
\begin{align*}
    &\frac{d}{dq_i}\frac{1}{n}\sum_{j=1}^{n} \sum_{i=1}^{m} h^i(X_j) l_{\alpha}(q_i, S_j)=0\\
    &\Longleftrightarrow \frac{1}{n}\sum_{j=1}^{n} h^i(X_j) \left[\1\left[S_j\leq q_i\right] - (1-\alpha)\right]=0
\end{align*}
This means for every $h$ the optimal $q_i$ should be essentially a weighted quantile of $S_1, S_2, \cdots, S_n$. This results in $q_{i}^*\in[0, 1]$, for every $i \in [1, \cdots, m]$. Consequently, the rest of the proof tries to show a similar bound holds for all the $h \in \mathcal{H}, \bq \in \mathcal{V}$, where $\mathcal{V}=\{\bq\in \mathcal{R}^m|\quad 0\leq q_i\leq 1 \quad\forall i\in\{1, 2, \cdots, m\}\}$. To do so we look at $\varepsilon$-net sets over $\mathcal{H}$ and $\mathcal{V}$. A union bound will give us a bound over the $\varepsilon$-net sets, and then leveraging the $\varepsilon$-net properties we can argue a bound over $\mathcal{H}$ and $\mathcal{V}$.Let's continue with with proving the following lemma that will pave the way for the rest of arguments,
\begin{lemma}\label{difflem}
For every $h_1, h_2 \in \mathcal{H}$ and $\boldsymbol{q^1}, \boldsymbol{q^2} \in \mathcal{V}$ we have,
\begin{align*}
    |Z_{h_1,\boldsymbol{q^1}}-Z_{h_2,\boldsymbol{q^2}}|\leq 2||\boldsymbol{q^1}-\boldsymbol{q^2}||_{\infty} + 2 d(h_1, h_2).
\end{align*}
\end{lemma}
\begin{proof}
\begin{align}
    \nonumber
    |Z_{h_1,\boldsymbol{q^1}}-Z_{h_2,\boldsymbol{q^2}}|&=|Z_{h_1,\boldsymbol{q^1}}-Z_{h_1,\boldsymbol{q^2}}+Z_{h_1,\boldsymbol{q^2}}-Z_{h_2,\boldsymbol{q^2}}|\\ \label{dif}
    &\leq |Z_{h_1,\boldsymbol{q^1}}-Z_{h_1,\boldsymbol{q^2}}|+|Z_{h_1,\boldsymbol{q^2}}-Z_{h_2,\boldsymbol{q^2}}|
\end{align}
Now we bound each term separately,\\
\textbf{First term:}
\begin{align*}
    |Z_{h_1,\boldsymbol{q^1}}-Z_{h_1,\boldsymbol{q^2}}|&=\left|\frac{1}{n} \sum_{j=1}^n \sum_{i=1}^m h_1^i(X_j)\left(\ell_\alpha(q_{i}^1, S_j) - \ell_\alpha(q_{i}^2, S_j)\right) -\underset{(X, S) \sim \mathcal{D}}{\mathbb{E}}\left[\sum_{i=1}^m h_1^i(X)\left(\ell_\alpha(q_{i}^1, S) - \ell_\alpha(q_{i}^2, S)\right)\right]\right|\\
    &\stackrel{(a)}{\leq} \left|\frac{1}{n} \sum_{j=1}^n \sum_{i=1}^m h_1^i(X_j)\left(\ell_\alpha(q_{i}^1, S_j) - \ell_\alpha(q_{i}^2, S_j)\right)\right| + \left|\underset{(X, S) \sim \mathcal{D}}{\mathbb{E}}\left[\sum_{i=1}^m h_1^i(X)\left(\ell_\alpha(q_{i}^1, S) - \ell_\alpha(q_{i}^2, S)\right)\right]\right|\\
    &\stackrel{(b)}{\leq} \frac{1}{n} \sum_{j=1}^n \sum_{i=1}^m h_1^i(X_j)\left|\ell_\alpha(q_{i}^1, S_j) - \ell_\alpha(q_{i}^2, S_j)\right| + \underset{(X, S) \sim \mathcal{D}}{\mathbb{E}}\left[\sum_{i=1}^m h_1^i(X)\left|\ell_\alpha(q_{i}^1, S) - \ell_\alpha(q_{i}^2, S)\right|\right]\\
    &\stackrel{(c)}{\leq} \frac{1}{n} \sum_{j=1}^n \sum_{i=1}^m h_1^i(X_j)\left|q_{i}^1 - q_{i}^2\right| + \underset{(X, S) \sim \mathcal{D}}{\mathbb{E}}\left[\sum_{i=1}^mh_1^i(X)\left|q_{i}^1 -q_{i}^2\right|\right]\\
    &\stackrel{(d)}{\leq} 2||\boldsymbol{q^1}-\boldsymbol{q^2}||_{\infty}
\end{align*}
Where (a) and (b) are triangle inequalities and (c) is followed by the fact that pinball loss is 1-lipschitz in terms of it's first argument(look at Lemma \ref{LipLem}). (d) is also by the definition of $||\boldsymbol{q^1}-\boldsymbol{q^2}||_{\infty}$.\\
\textbf{Second term:}
\begin{align*}
    |Z_{h_1\boldsymbol{q^2}}-Z_{h_2\boldsymbol{q^2}}|&=\left|\frac{1}{n} \sum_{j=1}^n \sum_{i=1}^m \left(h_1^i(X_j)-h_2^i(X_j)\right)l_\alpha(q_{i}^2, S_j) -\underset{(X, S) \sim \mathcal{D}}{\mathbb{E}}\left[\sum_{i=1}^m \left(h_1^i(X)-h_2^i(X)\right)l_\alpha(q_{i}^2, S)\right]\right|\\
    & \stackrel{(a)}{\leq} \left|\frac{1}{n} \sum_{j=1}^n \sum_{i=1}^m \left(h_1^i(X_j)-h_2^i(X_j)\right)l_\alpha(q_{i}^2, S_j)\right| +\left|\underset{(X, S) \sim \mathcal{D}}{\mathbb{E}}\left[\sum_{i=1}^m \left(h_1^i(X)-h_2^i(X)\right)l_\alpha(q_{i}^2, S)\right]\right|\\
    &\stackrel{(b)}{\leq} \frac{1}{n} \sum_{j=1}^n \sum_{i=1}^m \left|\left(h_1^i(X_j)-h_2^i(X_j)\right)l_\alpha(q_{i}^2, S_j)\right| +\underset{(X, S) \sim \mathcal{D}}{\mathbb{E}}\left[\sum_{i=1}^m \left|\left(h_1^i(X)-h_2^i(X)\right)l_\alpha(q_{i}^2, S)\right|\right]\\
    &\stackrel{(c)}{\leq} \frac{1}{n} \sum_{j=1}^n \sum_{i=1}^m \left|h_1^i(X_j)-h_2^i(X_j)\right|\left| q_{i}^2\right| +\underset{(X, S) \sim \mathcal{D}}{\mathbb{E}}\left[\sum_{i=1}^m \left|h_1^i(X)-h_2^i(X)\right|\left|q_{i}^2\right|\right]\\
    &\stackrel{(d)}{\leq} 2 d(h_1, h_2),
\end{align*}
where (a) and (b) are triangle inequalities and (c) is followed by the fact that pinball loss is 1-lipschitz in terms of it's first argument(look at Lemma \ref{LipLem}). (d) is also by the definition of $d(h_1, h_2)$.
Plugging back to the \eqref{dif} concludes the Lemma \eqref{difflem}.
\end{proof}
Continuing the proof of Theorem \ref{thm2}, let's say $\varepsilon_1$-net$(\mathcal{H})$ and $\varepsilon_2$-net$(\mathcal{V})$ are two minimal $\epsilon$-net sets ($\varepsilon$-net sets with minimum cardinality). Then applying a union bound we have, with probability $1-\delta$,
\begin{align}
    \left|Z_{h,\bq}\right| \leq \sqrt{\frac{\ln \left(\frac{2\mathcal{N}(\mathcal{H}, d, \epsilon_1)\mathcal{N}(\mathcal{V}, ||.||_\infty, \epsilon_2)}{\delta}\right)}{n}}\quad\quad \forall h\in \varepsilon_1-\text{net}(\mathcal{H}),\quad \forall \bq\in \varepsilon_2-\text{net}(\mathcal{V}).
\end{align}

Now we can do the following calculation for any arbitrary $h \in \mathcal{H}$ and $\bq\in \mathcal{V}$. If $\Tilde{h} \in$ $\varepsilon_1$-net$(\mathcal{H})$ and $\boldsymbol{\tilde{q}}\in $ $\varepsilon_2$-net$(\mathcal{V})$ such that $d(h, \Tilde{h})\leq \epsilon_1$ and $||\bq-\boldsymbol{\tilde{q}}||_{\infty}\leq \epsilon_2$ then we have,
\begin{align*}
   |Z_{h,\bq}| \leq |Z_{h,\bq}-Z_{\Tilde{h},\boldsymbol{\tilde{q}}}| + |Z_{\Tilde{h},\boldsymbol{\tilde{q}}}|
   & \leq 2||\bq-\boldsymbol{\tilde{q}}||_\infty + 2 d(h, \tilde{h})+ |Z_{\Tilde{h},\boldsymbol{\tilde{q}}}|\\
   & \leq 2\epsilon_1 + 2 \epsilon_2 + \sqrt{\frac{\ln \left(\frac{2\mathcal{N}(\mathcal{H}, d, \epsilon_1)\mathcal{N}(\mathcal{V}, ||.||_\infty, \epsilon_2)}{\delta}\right)}{n}}.
\end{align*}
Where the second inequality follows from Lemma \ref{difflem} and the third one is based on the fact that $\Tilde{h} \in net_{\epsilon_1}(\mathcal{H})$ and $\boldsymbol{\tilde{q}}\in net_{\epsilon_2}(\mathcal{V})$. Now putting $\epsilon_1=\epsilon_2=\frac{1}{n}$ we have,
\begin{align*}
    |Z_{h,\bq}| &\leq \frac{2}{n} + \frac{2}{n} + \sqrt{\frac{\ln \left(\frac{2\mathcal{N}(\mathcal{V}, ||.||_\infty, \epsilon_2))\mathcal{N}(\mathcal{H}, d, \frac{1}{n})}{\delta}\right)}{n}}\\
    &\leq 5 \sqrt{\frac{\ln \left(\frac{2\mathcal{N}(\mathcal{V}, ||.||_\infty, \epsilon_2)\mathcal{N}(\mathcal{H}, d, \frac{1}{n})}{\delta}\right)}{n}}\\
    &= 5 \sqrt{\frac{\ln\left( \frac{2}{\delta}\right) + \ln\left( \mathcal{N}(\mathcal{V}, ||.||_\infty, \epsilon_2)\right) + \ln \left(\mathcal{N}(\mathcal{H}, d, \frac{1}{n})\right)}{n}}
\end{align*}
where the second inequality happens for sufficiently large n.\\
As a result of our calculations, we know have the following two inequalities hold at the same time with probability $1-\delta$,
\begin{align*}
    |Z_{h^*,\bq^*}|\leq 5 \sqrt{\frac{\ln\left( \frac{2}{\delta}\right) + \ln\left( \mathcal{N}(\mathcal{V}, ||.||_\infty, \epsilon_2)\right) + \ln \left(\mathcal{N}(\mathcal{H}, d, \frac{1}{n})\right)}{n}}\\
    |Z_{h^\infty,\bq^\infty}|\leq 5 \sqrt{\frac{\ln\left( \frac{2}{\delta}\right) + \ln\left( \mathcal{N}(\mathcal{V}, ||.||_\infty, \epsilon_2)\right) + \ln \left(\mathcal{N}(\mathcal{H}, d, \frac{1}{n})\right)}{n}}
\end{align*}
This leads to,
\begin{equation}\label{fin}
    \left|\underset{(X, S) \sim \mathcal{D}}{\mathbb{E}}\left[\sum_{i=1}^m {h^*}^i(X)\ell_\alpha({q_{i}^*}, S)\right]-\underset{(X, S) \sim \mathcal{D}}{\mathbb{E}}\left[\sum_{i=1}^m {h^\infty}^i(X)\ell_\alpha(q^\infty_i, S)\right]\right|\leq 10 \sqrt{\frac{\ln\left( \frac{2}{\delta}\right) + \ln\left( \mathcal{N}(\mathcal{V}, ||.||_\infty, \epsilon_2)\right) + \ln \left(\mathcal{N}(\mathcal{H}, d, \frac{1}{n})\right)}{n}}
\end{equation}
Now putting everything together,
\begin{align*}
    \left|\underset{(X, S) \sim \mathcal{D}}{\mathbb{E}}\left[\sum_{i=1}^m {h^*}^i(X)\ell_\alpha({q_{i}^*}, S)\right] - 
    \underset{(X, S) \sim \mathcal{D}}{\mathbb{E}}\left[\ell_{\alpha}(q_{1-\alpha}(X), S)\right]\right| 
    &\leq \Bigg|\underset{(X, S) \sim \mathcal{D}}{\mathbb{E}}\left[\sum_{i=1}^m {h^*}^i(X)\ell_\alpha({q_{i}^*}, S)\right] \\
    &\quad - \underset{(X, S) \sim \mathcal{D}}{\mathbb{E}}\left[\sum_{i=1}^m {h^\infty}^i(X)\ell_\alpha(q^\infty_i, S)\right]\Bigg| \\
    &+ \Bigg|\underset{(X, S) \sim \mathcal{D}}{\mathbb{E}}\left[\sum_{i=1}^m {h^\infty}^i(X)\ell_\alpha(q^\infty_i, S)\right] \\
    &\quad - \underset{h \in \mathcal{X}^{\Delta_m}, \bq \in \mathbb{R}^m}{\text{min}} \underset{(X, S) \sim \mathcal{D}}{\mathbb{E}}\left[\sum_{i=1}^m h^i(X)\ell_\alpha(q_i, S)\right]\Bigg|\\
    &+ \Bigg|\underset{h \in \mathcal{X}^{\Delta_m}, \bq \in \mathbb{R}^m}{\text{min}} \underset{(X, S) \sim \mathcal{D}}{\mathbb{E}}\left[\sum_{i=1}^m h^i(X)\ell_\alpha(q_i, S)\right] \\
    &\quad - \underset{(X, S) \sim \mathcal{D}}{\mathbb{E}}\left[\ell_{\alpha}(q_{1-\alpha}(X), S)\right]\Bigg|\\
    &\leq 10 \sqrt{\frac{\ln\left( \frac{2}{\delta}\right) + \ln\left( \mathcal{N}(\mathcal{V}, ||.||_\infty, \epsilon_2)\right) + \ln \left(\mathcal{N}(\mathcal{H}, d, \frac{1}{n})\right)}{n}}\\
    & + 2\sqrt{\frac{\text{var}(q_{1-\alpha}(X))}{m}} + \lambda_\mathcal{H},
\end{align*}
where the last inequality follows from \eqref{fin}, the definition of realizability gap \eqref{gap}, and Theorem \ref{thm1}. Lemma \ref{covering} concludes the proof by proving an upperbound on the $\mathcal{N}(\mathcal{V}, ||.||_\infty, \epsilon_2)$.

\end{proof}

\section{Additional Figures}
In this section you can find additional plots associated with the experiments.
\begin{figure*}[ht]
\centering
\includegraphics[width=.49\textwidth]{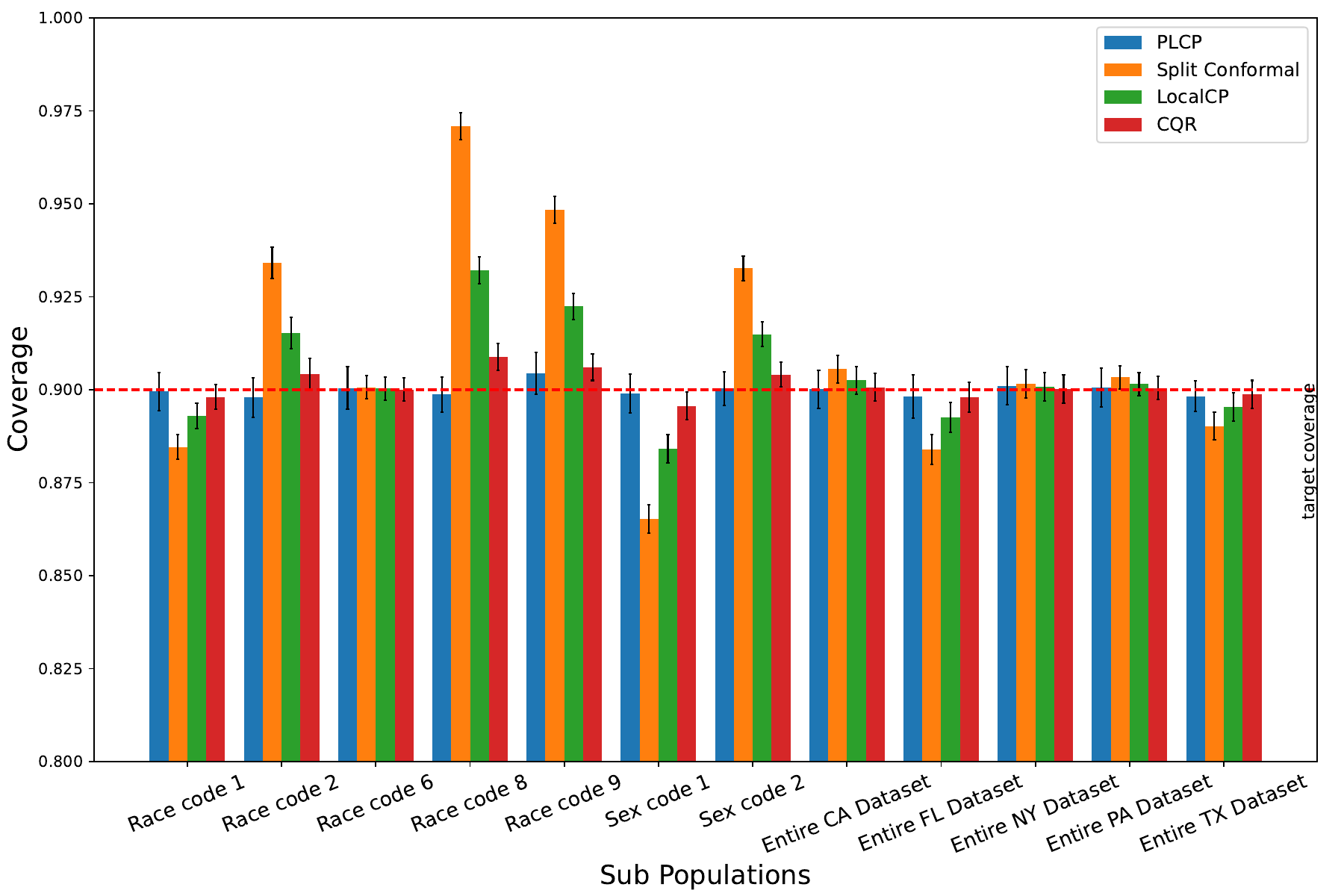}
\hfill
\includegraphics[width=.49\textwidth]{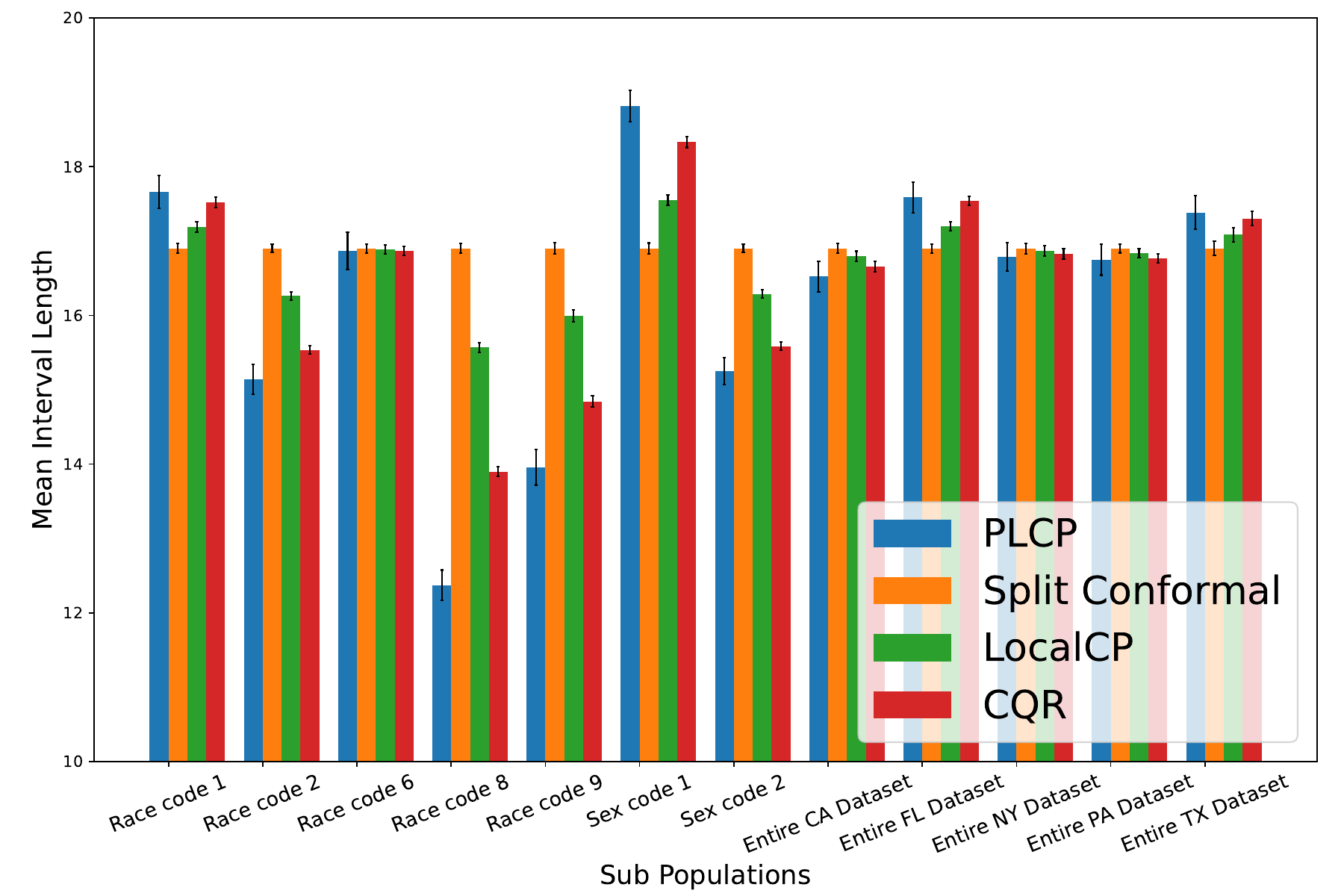}

\includegraphics[width=.4951\textwidth]{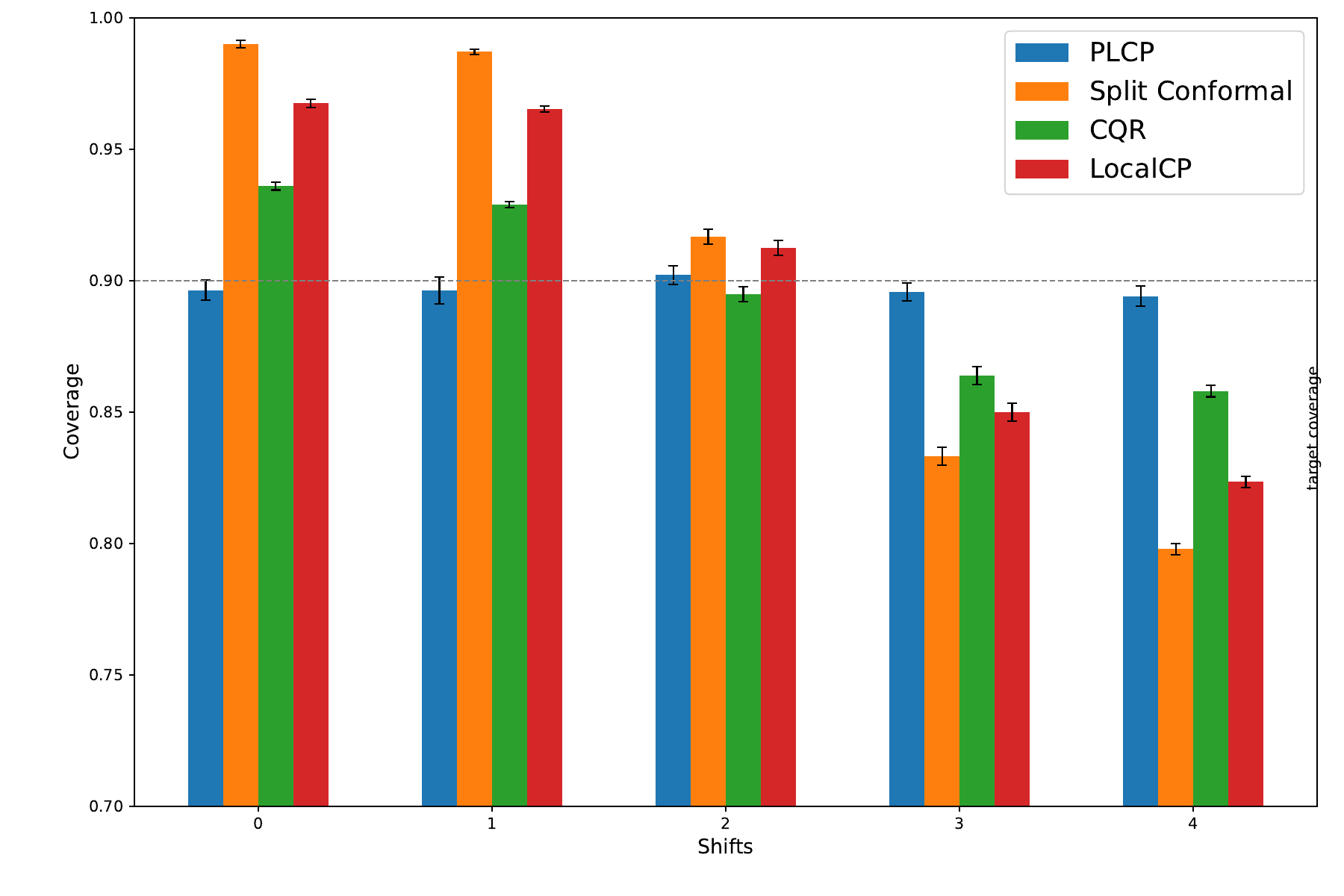}
\hfill
\includegraphics[width=.4951\textwidth]{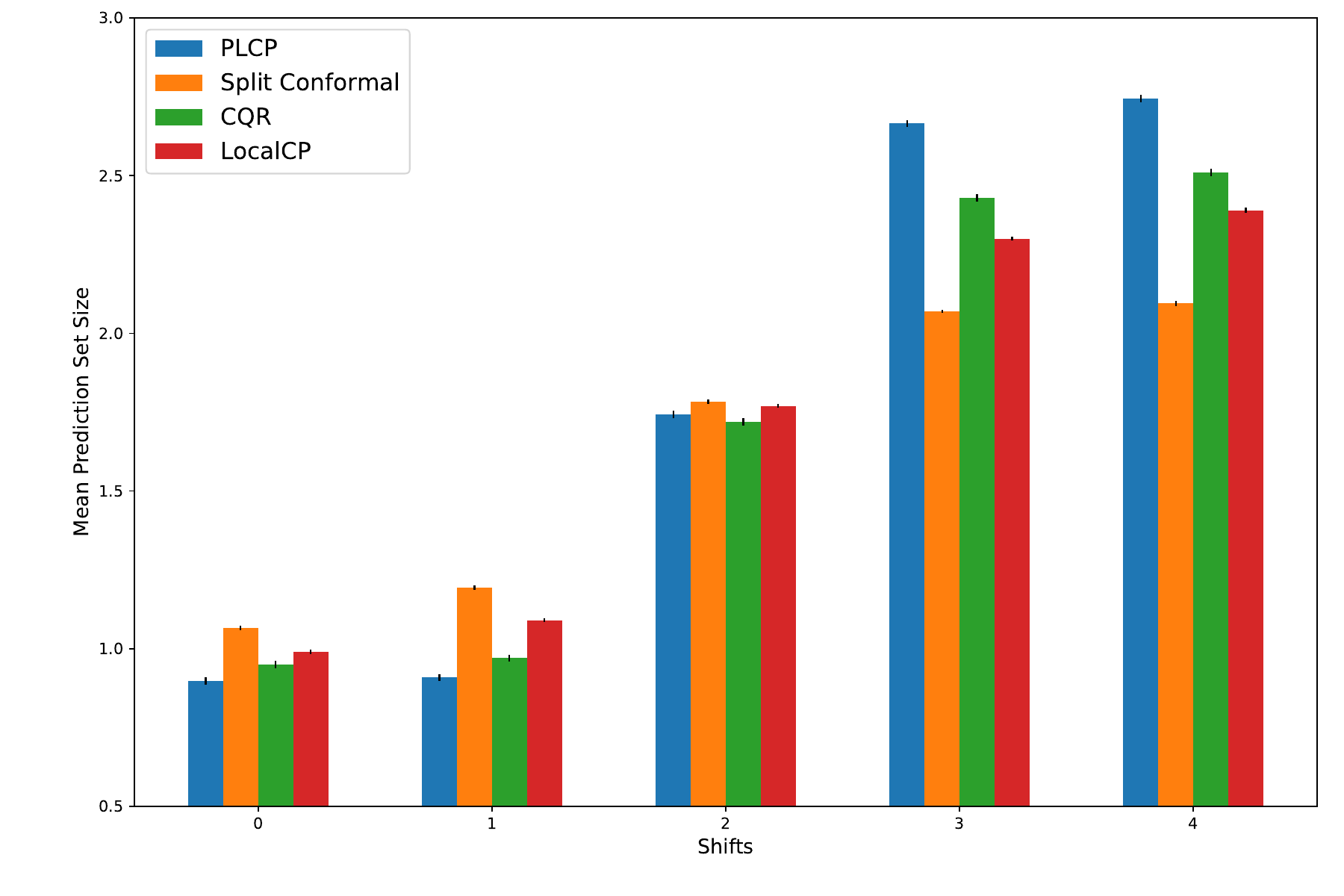}
\caption{Left-hand-side plots show coverage and right-hand-side plots show mean prediction set size. Row 1: US Census Data; Row 2: MNIST with Gaussian~Blur.}
\label{fig:combined_experiments1_reb}
\end{figure*}

\begin{figure*}[ht] 
\vskip 0.2in
\centering
\includegraphics[width=\textwidth]{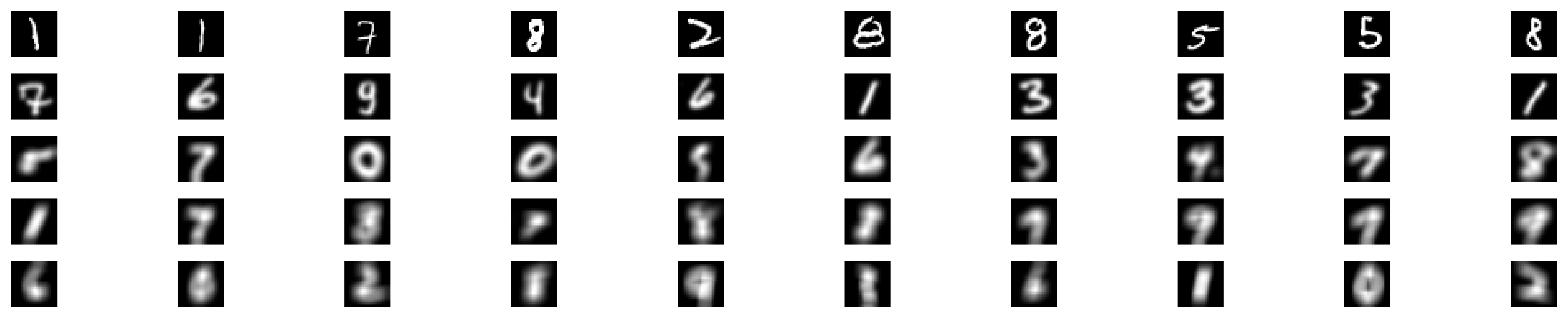}
\caption{Sample images from 5 groups with increasing levels of gaussian blur applied from top to bottom.} 
\label{fig:blured}
\vskip -0.2in
\end{figure*}


\end{document}